\pdfoutput=1
\documentclass[final,3p,times]{elsarticle}
\usepackage[T1]{fontenc}
\usepackage{amsmath,amssymb,amsthm,booktabs,array,graphicx,enumitem,xcolor}
\usepackage[hypertexnames=false,colorlinks=true,linkcolor=blue!70!black,citecolor=blue!70!black,urlcolor=blue!70!black,pdfauthor={Zeyu Jia, Bo Yao, Zihao Yang},pdftitle={Clause-Factored Bilateral State Grounding and Scoped Passivity for Counterfactual Open-World Deductive Reasoning},pdfcreator={LaTeX},pdfproducer={pdfTeX}]{hyperref}
\setlist{itemsep=2pt,topsep=3pt,parsep=0pt,partopsep=0pt}
\newtheorem{proposition}{Proposition}

\theoremstyle{remark}

\journal{Knowledge-Based Systems}

\begin{document}

\begin{frontmatter}

\title{Clause-Factored Bilateral State Grounding and Scoped Passivity for Counterfactual Open-World Deductive Reasoning}

\author[aff1,aff2]{Zeyu Jia\corref{cor1}}
\ead{jiazeyu@tju.edu.cn}
\author[aff1]{Bo Yao}
\ead{victory.bo.yao@gmail.com}
\author[aff1]{Zihao Yang}
\ead{zihaoy1018@tmu.edu.cn}

\cortext[cor1]{Corresponding author.}

\affiliation[aff1]{organization={School of Biomedical Engineering and Technology, Tianjin Medical University, No. 22 Qixiangtai Road, Heping District},
            city={Tianjin},
            postcode={300070},
            country={China}}

\affiliation[aff2]{organization={Medical School, Tianjin University, No. 92 Weijin Road, Nankai District},
            city={Tianjin},
            postcode={300072},
            country={China}}

\begin{abstract}
Open-World Assumption (OWA) deduction requires distinguishing provable facts (\texttt{True}), refutations (\texttt{False}), and unprovable statements (\texttt{Unknown}) across counterfactual edits and unseen domains. Sequence and latent-recurrent transformers rely on distributional shortcuts, whereas chain-of-thought language models incur high autoregressive latency. We decompose OWA deduction into a clause-factored bilateral grounder ($F$) and a deductive deliberation operator ($B$), forming $B \circ F$ in \textbf{Typed Edge-Factorized Bilateral Port-Hamiltonian Deliberation (TE-BPHD)}. Trained solely on three-way query labels without parsers or proof-depth annotations, \textbf{TE-BPHD} ($37.85\text{M}$ parameters) maps clauses into a $1{,}440$-dimensional Belnap space and propagates support through passivity-balanced implicit Port-Hamiltonian rule diodes. Across $17{,}036$ locked \texttt{ProofWriter} test queries, \textbf{TE-BPHD} reaches $79.17 \pm 5.48\%$ \texttt{depth-5} counterfactual consistency ($+23.83$ to $+38.67$ percentage points over \textbf{Seq2Seq T5}, \textbf{Flan-T5}, and \textbf{Recurrent-Depth T5}) and $94.95 \pm 1.23\%$ out-of-distribution (OOD) accuracy across $5{,}270$ six-domain queries ($92.54 \pm 3.38\%$ macro-mean). On a paired $120$-query OOD slice, \textbf{TE-BPHD} attains $93.61 \pm 4.59\%$ versus $79.17\%$ for $3$-shot \textbf{Qwen3-1.7B} under a $2{,}048$-token reasoning cap ($+14.44$ percentage points) at $9.44$ versus $14{,}878.4\,\text{ms/query}$ batch-$16$ \texttt{NVIDIA L4} throughput. Operator interventions show that symmetric single-rule passivity prevents positive two-cycle modes seen in \textbf{Gain-Matched Unilateral Horn} ($+50.00$ percentage points counterfactual consistency), whereas unit-interval \textbf{Grounded Product-T-Norm Logic Tensor Networks (LTN)} achieve cycle-safe OOD transfer ($94.50 \pm 0.87\%$) without linear solves, and multi-fact \texttt{ParaRules} prose exposes an upstream grounding trade-off ($1.33 \pm 1.76\%$ zero-shot vs.\ $34.11 \pm 5.01\%$ with five-slot optimal-transport joint grounding).
\end{abstract}


\end{frontmatter}

\section{Introduction}
\label{sec:intro}

Multi-hop deductive reasoning over explicit natural language theories demands compositional fidelity under the Open-World Assumption (OWA), where the absence of a proof yields \texttt{Unknown} rather than \texttt{False}~\cite{belnap1977useful,clark2020transformers,tafjord2021proofwriter}. Recent approaches scale test-time inference either by generating discrete chain-of-thought token traces~\cite{snell2024scaling,saparov2023language,yang2025qwen3} or by iterating continuous latent states~\cite{dehghani2018universal,hao2024training,geiping2025scaling}. Despite strong standard-split accuracy, neural reasoners face a persistent tension between surface shortcut exploitation, cross-domain fragility, and test-time computational cost~\cite{bergen2021systematic,yang2025socrates}. When upstream root premises are counterfactually negated or deleted, token-level sequence-to-sequence and latent recurrent-depth transformers drop sharply in familywise consistency, while autoregressive reasoning models require thousands of generated tokens per query to track open-world unprovability.

Understanding why continuous neural reasoners succeed or fail requires separating two often-conflated stages of a neuro-symbolic pipeline $B \circ F$: the upstream grounding frontend $F$, which maps raw natural language sentences into structured state variables, and the downstream deliberation backend $B$, which propagates truth support across multi-hop relational rules. Prior continuous dynamical and neuro-symbolic formulations leave three mechanistic questions unresolved. First, to what extent do counterfactual and out-of-distribution (OOD) gains stem from explicit clause-factored bilateral grounding ($F$) versus the specific choice of dynamical solver ($B$), and what multi-fact grounding trade-offs arise when single sentences pack multiple crowdsourced facts without parser supervision? Second, within continuous deductive operators ($B$), how does single-rule Port-Hamiltonian passivity~\cite{vanderschaft2014port,ortega2002interconnection,rusch2022graphcon} govern cyclic stability in unbounded linear implicit solvers relative to unit-interval ($[0,1]^{2N}$) fuzzy t-norm/s-norm saturation~\cite{badreddine2022logic,vankrieken2022analyzing}? Third, how do compact grounded state operators compare against modern reasoning language models under higher-budget ($2{,}048$-token cap) chain-of-thought decoding?

\begin{figure*}[t]
\centering
\includegraphics[width=\textwidth]{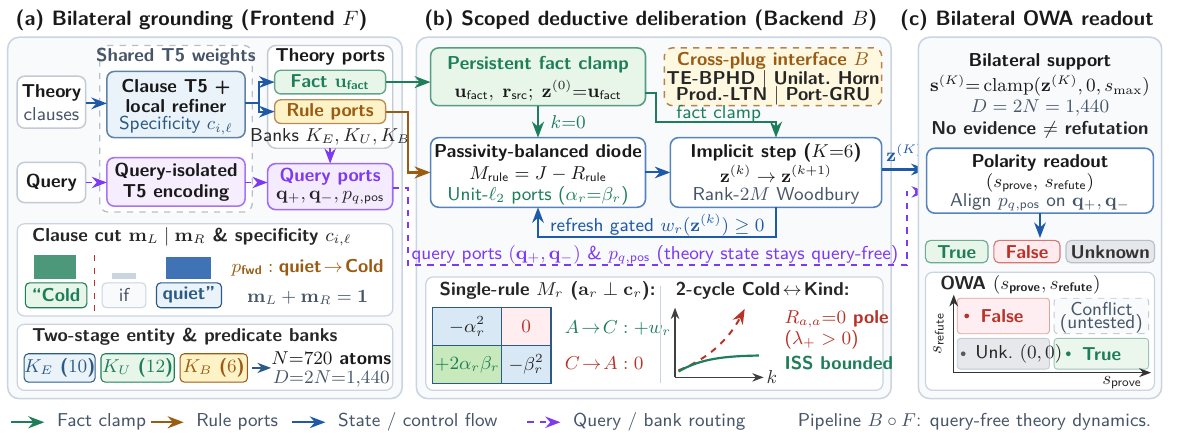}
\caption{\textbf{Modular $B \circ F$ Architecture of TE-BPHD.} \textbf{(a) Bilateral Clause Grounding (Frontend $F$):} Theory clauses and queries share a T5 encoder with contextual specificity gate $c_{i,\ell}$, a non-deleting interior cut $\mathbf{m}_L + \mathbf{m}_R = \mathbf{1}$ (illustrated on reversed conditional \emph{``Cold | if | quiet''} with orientation $p_{\text{fwd}}: \text{quiet}\!\to\!\text{Cold}$), and two-stage entity/predicate banks ($K_E, K_U, K_B$) to populate fact ports $\mathbf{u}_{\text{fact}}$, rule ports ($M=96$), and query-isolated readout ports ($\mathbf{q}_+, \mathbf{q}_-, p_{q,\text{pos}}$) over $N=720$ logical atoms ($D=1{,}440$). \textbf{(b) Scoped Deductive Deliberation (Backend $B$):} Persistent fact clamping drives passivity-balanced Port-Hamiltonian rule diodes ($M_{\text{rule}} = J - R_{\text{rule}}$ with unit-$\ell_2$ ports; orthonormal single-rule block $\mathbf{a}_r \perp \mathbf{c}_r$ shown at bottom-left) through $K=6$ rank-$2M$ Woodbury implicit steps while preventing two-cycle ($\text{Cold}\leftrightarrow\text{Kind}$) resolvent poles, and exposes a shared port interface for cross-plugging \textbf{Gain-Matched Unilateral Horn}, \textbf{Grounded Product-T-Norm LTN}, and \textbf{Grounded Port-GRU}. \textbf{(c) Bilateral OWA Readout:} Clamped support $\mathbf{s}^{(K)}$ aligns query polarity onto horizontal $s_{\text{prove}}$ and vertical $s_{\text{refute}}$ axes to separate \texttt{True} (bottom-right), \texttt{False} (top-left), and \texttt{Unknown} near $(0,0)$, with simultaneous activation marked \texttt{Conflict (untested)}.}
\label{fig:architecture}
\end{figure*}

To answer these questions rigorously, we introduce \textbf{Typed Edge-Factorized Bilateral Port-Hamiltonian Deliberation (TE-BPHD)}, depicted in Figure~\ref{fig:architecture}, and evaluate it across a development-locked $2\times 3$ factorial tournament, cross-plugged operator interventions, and a single-pass $17{,}036$-query locked test suite on official \texttt{ProofWriter V2020.12.3}~\cite{tafjord2021proofwriter,proofwriter2020dataset}. Trained end-to-end from three-way (\texttt{False}, \texttt{True}, \texttt{Unknown}) query labels without syntactic parsers or proof-depth supervision, \textbf{TE-BPHD} maps sentences into a $1{,}440$-dimensional Belnap truth space and propagates support across unit-$\ell_2$ rule ports via passivity-balanced implicit Port-Hamiltonian diodes. Our main contributions are fourfold:
\begin{itemize}
    \item \textbf{Clause-Factored Bilateral Grounding and Linguistic Regime Characterization}: Mapping single-clause theories into a $1{,}440$-dimensional Belnap atom-edge space yields $79.17 \pm 5.48\%$ \texttt{depth-5} counterfactual consistency and $94.95 \pm 1.23\%$ query-weighted ($92.54 \pm 3.38\%$ macro-mean) zero-shot accuracy across six \texttt{birds-electricity} OOD domains, surpassing converged \textbf{Seq2Seq T5}, \textbf{Flan-T5}, and \textbf{Recurrent-Depth T5} by $+23.83$ to $+38.67$ percentage points on counterfactuals and $+28.58$ to $+40.22$ percentage points on OOD ($p_{\text{Holm}}=0.0050$). On crowdsourced multi-fact \texttt{ParaRules} prose, composite five-slot optimal-transport grounding lifts joint counterfactual consistency from $12.67 \pm 4.70\%$ to $34.11 \pm 5.01\%$ ($p_{\text{Holm}}=0.0050$), whereas token-level \textbf{Recurrent-Depth T5} reaches $78.11 \pm 2.55\%$ in-distribution at a $20.82$ percentage point OOD cost.
    \item \textbf{Orthonormal Single-Rule Passivity vs.\ Bounded Fuzzy Saturation}: On an orthonormal single-rule port pair ($\mathbf{a}_r \perp \mathbf{c}_r$), achieving forward Modus Ponens conduction and zero single-rule converse coupling with $R_r \succeq 0$ satisfies $R_{a,a}R_{c,c} \ge f_r^2/4$ (Proposition~\ref{prop:sylvester}); omitting antecedent self-dissipation in \textbf{Gain-Matched Unilateral Horn} ($R_{a,a}=0$) preserves $100.00\%$ isolated fan-out amplitude yet induces a positive symmetric two-cycle eigenvalue $\lambda_+ = +w_r/2 > 0$ and a $-50.00$ percentage point test counterfactual drop ($p_{\text{Holm}}=0.0050$). Cross-plugging \textbf{Grounded Product-T-Norm LTN} shows that unit-interval ($[0,1]^{2N}$) Probabilistic-Sum S-norm saturation achieves cycle-safe OOD transfer ($94.50 \pm 0.87\%$, $p_{\text{Holm}}=1.00$) without linear solves, while both structured backends outperform \textbf{Grounded Port-GRU} ($79.98 \pm 13.27\%$ OOD, $p_{\text{Holm}}=0.0050$).
    \item \textbf{Pointwise Discrete $P$-Quadratic ISS and Unscaled Port Coupling}: We establish an exact discrete $P$-quadratic Input-to-State Stability (ISS) Lyapunov identity (uniformly bounded for bounded fixed $\mathbf{u}_{\text{fact}}$) and $\mathcal{O}(DM^2 + M^3)$ rank-$2M$ Woodbury equivalence (Proposition~\ref{prop:iss}), while showing that omitting composite Gram--Schmidt orthogonalization and out-degree scaling improves standard \texttt{depth-5} test accuracy by $+5.89$ percentage points over \textbf{Out-Degree-Scaled TE-BPHD} ($p_{\text{Holm}}=0.0050$) and threshold-free gating retains $77.33 \pm 4.91\%$ counterfactual consistency ($p_{\text{Holm}}=1.00$).
    \item \textbf{Higher-Budget Reasoning-LLM Accuracy--Throughput Frontier}: Under greedy decoding capped at $2{,}048$ reasoning tokens, \textbf{Qwen3-1.7B} lifts natural \texttt{depth-5} counterfactual trace closure from $\le 0.83\%$ (at $384$ tokens) to $64.17\%\text{--}75.83\%$ and $3$-shot counterfactual consistency to $50.00\%$; on the paired $120$-query OOD slice, \textbf{TE-BPHD} attains $93.61 \pm 4.59\%$ versus $79.17\%$ for $3$-shot \textbf{Qwen3-1.7B} ($+14.44$ percentage points, $p_{\text{Holm}}=0.0136$) at $9.44$ versus $14{,}878.4\,\text{ms/query}$ \texttt{batch=16} \texttt{NVIDIA L4} throughput.
\end{itemize}

Section~\ref{sec:related} situates our framework within prior paradigms. Sections~\ref{sec:frontend} and~\ref{sec:backend} formalize the grounding frontend $F$ and deductive backends $B$. Section~\ref{sec:protocol} details the firewall-locked evaluation protocol, Section~\ref{sec:results} presents empirical and mechanistic results, and Section~\ref{sec:discussion} synthesizes scope boundaries and conclusions.

\begin{table*}[t]
\centering
\footnotesize
\setlength{\tabcolsep}{3.2pt}
\caption{\textbf{Conceptual and Structural Taxonomy of Open-World Deductive Reasoning Paradigms.}}
\label{tab:taxonomy}
\begin{tabular}{@{}llllll@{}}
\toprule
\textbf{Paradigm \& Models} & \textbf{Grounding} & \textbf{4-Valued OWA} & \textbf{Converse} & \textbf{Cyclic Bound} & \textbf{Step Cost} \\
\midrule
Sequence / CoT (\textbf{Seq2Seq}, \textbf{Flan-T5}, \textbf{Qwen3}) & Token implicit & Label / prompt & Unconstrained & Budget dependent & $\mathcal{O}(L_{\text{seq}}^2 + T_{\text{CoT}})$ \\
Latent Recurrent (\textbf{Recurrent-Depth T5}) & Token implicit & Readout head & Unconstrained & Unbounded residual & $\mathcal{O}(L_{\text{seq}}^2)$ \\
Gated Graph (\textbf{Grounded Port-GRU}) & Atom-edge & Bilateral & Learned MLP & Bounded $\tanh$ & $\mathcal{O}(DM)$ \\
Fuzzy Chaining (\textbf{Product-T-Norm LTN}) & Atom-edge & $[0,1]^{2N}$ & Directed write & S-norm in $[0,1)^{2N}$ & $\mathcal{O}(DM)$ \\
Non-Passive Implicit (\textbf{Unilateral Horn}) & Atom-edge & $\mathbb{R}^{2N}$ & $0$ ($\mathbf{a}_r \perp \mathbf{c}_r$) & Pole ($\lambda_+ = {+}w_r/2$) & $\mathcal{O}(DM^2 + M^3)$ \\
\textbf{TE-BPHD (Proposed)} & Atom-edge & $\mathbb{R}^{2N}$ & $0$ ($\mathbf{a}_r \perp \mathbf{c}_r$) & $P$-quadratic ISS & $\mathcal{O}(DM^2 + M^3)$ \\
\bottomrule
\end{tabular}
\vspace{2pt}\par\noindent\parbox{\textwidth}{\scriptsize \textit{Notes:} Comparison of neural, latent-recurrent, differentiable-fuzzy, and implicit Port-Hamiltonian architectures across grounding granularity, Open-World Assumption (OWA) bilateral support representation (supervised via three-way \texttt{True}/\texttt{False}/\texttt{Unknown} labels), single-rule converse control on orthonormal ports ($\mathbf{a}_r \perp \mathbf{c}_r$), cyclic stability, and per-step computational cost ($D=2N=1{,}440$ bilateral state coordinates, $M=96$ grounded rule ports).}
\end{table*}

\section{Related Work}
\label{sec:related}

\textbf{Neuro-Symbolic and Differentiable Fuzzy Reasoning.} Classical differentiable deductive systems---including TensorLog~\cite{cohen2020tensorlog}, NeuralLP~\cite{yang2017differentiable}, $\partial$ILP~\cite{evans2018learning}, DeepProbLog~\cite{manhaeve2018deepproblog}, Neural Theorem Provers~\cite{rocktaschel2017end}, and Neural Logic Machines~\cite{dong2019neural}---assume pre-parsed symbolic predicates and entity constants, whereas hybrid semantic-parsing pipelines such as Logic-LM~\cite{pan2023logiclm} and LINC~\cite{olausson2023linc} prompt large language models to translate text into formal logic programs executed by external symbolic theorem provers. Recent knowledge-based reasoning studies extend logic query embeddings and reinforcement-mined Horn rules over structured knowledge graphs~\cite{wu2024neurosymbolic,chen2022rulemining,dai2025question,galkin2024towards}, while Logic Tensor Networks (LTN)~\cite{badreddine2022logic} and Logical Neural Networks (LNN)~\cite{riegel2020logical} relax crisp Modus Ponens via continuous t-norms and s-norms whose gradient geometry was analyzed systematically by van Krieken et al.~\cite{vankrieken2022analyzing}. However, standard fuzzy provers do not specify how to extract four-valued Belnap support~\cite{belnap1977useful} directly from unparsed natural language clauses under query-only supervision without external solvers. By cross-plugging a faithful Product-T-Norm / Probabilistic-Sum LTN forward chainer directly onto our learned bilateral ports, we delineate the empirical regime where bounded $[0,1]^{2N}$ S-norm saturation suffices versus where linear implicit passivity governs stability (Table~\ref{tab:taxonomy}).

\textbf{Sequence, Chain-of-Thought, and Latent Recurrent Transformers.} Fine-tuned encoder-decoders such as RuleTaker~\cite{clark2020transformers}, ProofWriter~\cite{tafjord2021proofwriter}, and instruction-tuned Flan-T5~\cite{raffel2020exploring,chung2024scaling} map raw text directly to truth labels or proof strings, while test-time compute scaling generates explicit chain-of-thought rationales~\cite{snell2024scaling,saparov2023language,han2024folio,yang2025qwen3}. Recent mechanistic studies show that implicit multi-hop reasoning in language models often relies on latent shortcut correlations~\cite{yang2025socrates}; in parallel, because token-level self-attention mixes premise and rule-head tokens globally, sequence models exhibit brittle generalization under systematic relational shifts and counterfactual root-fact edits~\cite{bergen2021systematic,tafjord2021proofwriter}, whereas long reasoning traces incur multi-second latency per query. Continuous latent-deliberation architectures---from Universal Transformers~\cite{dehghani2018universal} and PonderNet~\cite{graves2016adaptive,banino2021pondernet} to COCONUT~\cite{hao2024training} and Recurrent-Depth Transformers~\cite{geiping2025scaling}---iterate shared transformer blocks without emitting intermediate tokens. Our experiments show that token-level recurrent depth excels when jointly supervised on crowdsourced multi-fact prose, yet lags clause-factored bilateral grounding by nearly $29$ percentage points on zero-shot cross-domain transfer.

\textbf{Port-Hamiltonian, Implicit Graph, and Equilibrium Networks.} Continuous-depth and implicit equilibrium models~\cite{chen2018neural,bai2019deep,gu2020implicit,winston2020monotone} trade explicit unrolling depth for fixed-point or differential-equation structure, recently generalized to non-Euclidean monotone operators~\cite{davydov2024non}. Hamiltonian and Port-Hamiltonian neural networks~\cite{greydanus2019hamiltonian,vanderschaft2014port,desai2021port} and Graph-Coupled Oscillator Networks (GraphCON)~\cite{rusch2022graphcon} preserve bounded energy via skew-symmetric interconnection $J = -J^\top$ and positive semi-definite damping $R \succeq 0$. Applied naively to logical implication, pure skew-symmetry induces reciprocal backward truth leakage from consequents to antecedents, whereas non-reciprocal active systems~\cite{fruchart2021non} can amplify cyclic modes if damping is omitted. Unlike prior Port-Hamiltonian formulations that claim broad superiority over discrete logic, we characterize the Sylvester $2\times 2$ passivity boundary separating stable linear-implicit deduction from cyclic resolvent poles.

\section{Clause-Factored Bilateral State Grounding (\texorpdfstring{Frontend $F$}{Frontend F})}
\label{sec:frontend}

\subsection{Open-World Deduction Setup and Typed Bilateral State Space}
Each deductive instance consists of a natural language theory $\mathcal{T} = \{s_1, \dots, s_L\}$ of $L$ declarative sentences and a natural language query $q$, paired only with a three-way Open-World Assumption target $y \in \{\texttt{False}, \texttt{True}, \texttt{Unknown}\}$. Under OWA, \texttt{True} indicates that $q$ is entailed by $\mathcal{T}$, \texttt{False} indicates that the logical negation $\neg q$ is entailed by $\mathcal{T}$, and \texttt{Unknown} indicates that neither $q$ nor $\neg q$ is derivable. Representing OWA faithfully requires separating lack of evidence from explicit refutation~\cite{belnap1977useful}.

We construct a typed finite relational domain comprising $E = 10$ entity slots, $P_U = 12$ unary predicate slots, and $P_B = 6$ binary predicate slots, yielding $N = E P_U + E^2 P_B = 720$ ground logical atoms. Assigning each atom independent positive (affirmative) and negative (refutational) support coordinates defines a $D = 2N = 1{,}440$-dimensional bilateral state vector $\mathbf{z}^{(k)} = [\mathbf{z}_+^{(k)}; \mathbf{z}_-^{(k)}] \in \mathbb{R}^{2N}$ at deliberation iteration $k \in \{0, \dots, K\}$. Let $\Pi_{\text{neg}} \in \{0, 1\}^{2N \times 2N}$ denote the involutive polarity-swap permutation exchanging $\mathbf{z}_+$ and $\mathbf{z}_-$ ($\Pi_{\text{neg}}^2 = I_{2N}$). Near the origin, $(\mathbf{z}_+)_n \approx 0$ and $(\mathbf{z}_-)_n \approx 0$ encode Belnap's \emph{Neither} (\texttt{Unknown}), single-channel activation encodes \texttt{True} or \texttt{False}, and simultaneous activation corresponds to local contradiction (\emph{Both}); because \texttt{ProofWriter} provides three-way OWA supervision (\texttt{True}, \texttt{False}, \texttt{Unknown}), explicit supervision of contradictory (\emph{Both}) targets lies outside our benchmark scope.

To make these distinctions concrete, consider a five-sentence mini-theory $\mathcal{T}$ with base facts $s_1$: \emph{``Anne is quiet.''} and $s_2$: \emph{``Anne is rough.''}, reversed conditional rule $s_3$: \emph{``Rough things are cold if quiet.''}, and reciprocal rules $s_4$: \emph{``If someone is cold then they are kind.''} and $s_5$: \emph{``If someone is kind then they are cold.''} Under $\mathcal{T}$, forward chaining derives $\text{Kind}(\text{Anne})$ with positive support and near-zero refutation ($(z_+ > 0, z_- \approx 0) \mapsto \texttt{True}$); replacing $s_3$ with an explicit negative-head rule \emph{``Rough things are not kind if quiet.''} routes refutational support into $z_- > 0$ ($\texttt{False}$); and deleting $s_1$ or reversing the implication direction leaves $\text{Kind}(\text{Anne})$ unprovable near the bilateral origin $(z_+ \approx 0, z_- \approx 0) \mapsto \texttt{Unknown}$. Evaluating counterfactual faithfulness correspondingly requires answering an entire query family consistently: on \texttt{depth-5}, each \texttt{All-3} triplet ($200$ triplets $= 600$ queries across $6$ perturbation mechanisms) holds the query fixed across three edited theory contexts entailing \texttt{True}, \texttt{False}, and \texttt{Unknown} ($200{:}200{:}200$ class-balanced, with $100.0\%$ of labels re-derived via forward signed Horn closure so uniformly predicting \texttt{Unknown} scores $0.0\%$ triplet consistency), while \texttt{ParaRules} \texttt{All-4} pairs two intact polarity twins (\texttt{True}, \texttt{False}) with two premise-dropped \texttt{Unknown} twins. Meanwhile, the mutual implication pair $s_4, s_5$ ($\text{Cold}(x) \leftrightarrow \text{Kind}(x)$) forms a two-cycle feedback loop that stresses cyclic stability during multi-hop propagation.

\subsection{Contextual Single-Slot Clause Grounder}
In \textbf{TE-BPHD}, a shared \texttt{T5-small} encoder~\cite{raffel2020exploring} ($d_{\text{model}} = 512$) maps each sentence $s_i$ of length $T_i$ to contextual token states $\mathbf{H}_i = [\mathbf{h}_{i,1}, \dots, \mathbf{h}_{i,T_i}] \in \mathbb{R}^{T_i \times d_{\text{model}}}$, refined by local self-attention blocks without string templates or external syntactic parsers. To suppress closed-class syntactic scaffolding (such as determiners and conditional markers) while retaining open-class entities and predicates, each token receives a contextual specificity gate $c_{i,\ell} \in [0, 1]$ computed by comparing projected token representations against a rolling first-in-first-out buffer of $B_{\text{bg}} = 384$ contextualized clause embeddings ($24$ tokens per clause). This buffer updates exclusively during training and remains strictly frozen during evaluation ($0.0$ buffer mutation across evaluation batches, ensuring order-invariant inference).

A learned clause-type classifier assigns each sentence a fact probability $p_{\text{fact},i} \in [0, 1]$ versus rule probability $1 - p_{\text{fact},i}$. For candidate rules, a pointer splitter~\cite{vinyals2015pointer,dozat2017deep} predicts a non-deleting interior cut $\mathbf{m}_{i,L} + \mathbf{m}_{i,R} = \mathbf{1}_{T_i}$ together with a forward-vs.-reversed conditional orientation probability $p_{\text{fwd},i} \in [0, 1]$, enabling the grounder to align canonical conditionals (\emph{``If X then Y''}) and reversed conditionals (\emph{``Y if X''}) onto antecedent and consequent spans. Within each sub-clause span, role-specific attention heads extract subject, predicate, and object embeddings under an English active-voice subject--verb--object ordering bias, alongside a conjunctive arity probability $a_{r,2} \in [0, 1]$ for two-premise rule bodies ($a_{r,3} \equiv 0$ in single-slot \textbf{TE-BPHD}).

To bind surface phrases consistently across a theory without global vocabulary tables, theory-local key banks $K_E \in \mathbb{R}^{E \times d_{\text{key}}}$, $K_U \in \mathbb{R}^{P_U \times d_{\text{key}}}$, and $K_B \in \mathbb{R}^{P_B \times d_{\text{key}}}$ ($d_{\text{key}} = 256$) are populated in two stages: fact clauses first seed canonical entity and predicate keys via temperature-scaled cosine matching, after which rule clauses complete novel symbols not yet instantiated in facts. Tensor-product binding over $(K_E, K_U, K_B)$ with polarity routing $\Pi_{\text{neg}}$ maps fact clauses into a persistent non-negative fact-clamping vector $\mathbf{u}_{\text{fact}} \in \mathbb{R}_{\ge 0}^{2N}$ that anchors stated base facts throughout deliberation (analogous to a Norton current source in circuit theory). For rules, from an $18$-slot candidate rule buffer, the grounder selects up to $R_c = 14$ candidate rule clauses ($2R_c \le 28$ forward and reversed conditional orientations), instantiates universally quantified variables across the $E = 10$ entity fiber slots ($2R_c E \le 280$ candidate grounded instances), and prunes the top $M = \min(96, 2R_c E) = 96$ active grounded rule ports with two-premise port vectors $\{\mathbf{b}_{r,1}, \mathbf{b}_{r,2}\} \subset \mathbb{R}_{\ge 0}^{2N}$, raw consequent ports $\mathbf{c}_{\text{raw},r} \in \mathbb{R}_{\ge 0}^{2N}$, conjunctive arities $a_{r,2} \in [0, 1]$, and base rule weights $\bar{w}_r \ge 0$ ($r \in \{1, \dots, M\}$); empirically across development theories, $0.0\%$ exceed $14$ or $18$ predicted rules while $31.83\%$ reach the $M = 96$ port cap. The query $q$ is encoded independently against the frozen theory banks $(K_E, K_U, K_B)$ to produce query ports $\mathbf{q}_+, \mathbf{q}_- \in \mathbb{R}_{\ge 0}^{2N}$ and polarity probability $p_{q,\text{pos}} \in [0, 1]$, keeping theory dynamics $\mathbf{z}^{(k)}$ strictly query-independent so a single deliberation trajectory amortizes across multiple queries sharing $\mathcal{T}$.

\subsection{Stateless Multi-Slot Optimal-Transport Grounder}
Single-slot clause grounding assumes each sentence expresses at most one base fact or rule, an inductive bias stressed by crowdsourced prose where multiple facts are conjoined in a single sentence. To probe this upstream boundary directly, we also formulate a stateless \textbf{Multi-Slot OT Grounder}~\cite{locatello2020object,ye2025differential}. Replacing the rolling corpus-frequency buffer with a stateless token-specificity MLP, it applies differential multi-head cross-attention~\cite{ye2025differential} to extract $S_{\text{slot}} = 5$ parallel candidate fact slots per sentence via entropic Sinkhorn optimal transport~\cite{cuturi2013sinkhorn,zhang2023unlocking}, extends sub-clause rule partitioning to three-premise bodies ($\{\mathbf{b}_{r,1}, \mathbf{b}_{r,2}, \mathbf{b}_{r,3}\} \subset \mathbb{R}_{\ge 0}^{2N}$ with conjunctive arities $a_{r,2}, a_{r,3} \in [0, 1]$), and prunes the top $M = 96$ fiber-instantiated rule ports into the identical $D = 1{,}440$ bilateral port interface. When paired with downstream operators, its third premise $(\mathbf{b}_{r,3}, a_{r,3})$ extends the two-premise bottleneck via $\theta_{r,3} = a_{r,2}a_{r,3}\sigma(10(\min(s_{r,1}^{(k)}, s_{r,2}^{(k)}) - s_{r,3}^{(k)}))$, $\tilde{\mathbf{a}}_r^{(k)} \leftarrow (1 - \theta_{r,3})\tilde{\mathbf{a}}_r^{(k)} + \theta_{r,3}\mathbf{b}_{r,3}$, and three-way conjunctive gating $(1 - a_{r,3})\psi(s_{r,1}^{(k)})\psi(s_{r,2}^{(k)}) + a_{r,3}\psi(s_{r,1}^{(k)})\psi(s_{r,2}^{(k)})\psi(s_{r,3}^{(k)})$.

\subsection{Polarity-Equivariant Readout and Query-Only Training Objective}
After $K$ deliberation steps, states are clamped to bounded non-negative support $\mathbf{s}^{(K)} = \operatorname{clamp}(\mathbf{z}^{(K)}, 0, s_{\max})$ ($s_{\max} = 2.5$). Query-aligned proof and refutation scores are read out symmetrically via $s_{\text{prove}} = p_{q,\text{pos}} (\mathbf{q}_+^\top \mathbf{s}^{(K)}) + (1 - p_{q,\text{pos}})(\mathbf{q}_-^\top \mathbf{s}^{(K)})$ and $s_{\text{refute}} = p_{q,\text{pos}} (\mathbf{q}_-^\top \mathbf{s}^{(K)}) + (1 - p_{q,\text{pos}})(\mathbf{q}_+^\top \mathbf{s}^{(K)})$, which map linearly to three-way logits $(\ell_{\texttt{False}}, \ell_{\texttt{True}}, \ell_{\texttt{Unknown}})$. All models are trained end-to-end by minimizing $\mathcal{L} = \mathcal{L}_{\text{CE}}(y, \hat{\mathbf{y}}) + \lambda_{\text{marg}} \mathcal{R}_{\text{marg}}$, where $\mathcal{R}_{\text{marg}}$ is an unsupervised batch-marginal entropy regularizer preventing slot collapse across $(K_E, K_U, K_B)$ without accessing symbolic parses or proof depths.

\section{Scoped Deductive Deliberation Operators (\texorpdfstring{Backend $B$}{Backend B})}
\label{sec:backend}

Before formalizing the continuous update equations, we summarize how three core deductive operations map onto dynamical system primitives. First, \emph{persistent source clamping} ($\mathbf{u}_{\text{fact}}, \mathbf{r}_{\text{src}}$, analogous to a Norton current source with shunt conductance) keeps explicitly stated base facts anchored across multi-hop chaining rather than decaying away. Second, \emph{directional rule diodes} ($M_{\text{rule}} = J - R_{\text{rule}}$) combine skew-symmetric coupling with matched port dissipation so truth support flows forward from antecedents to consequents (Modus Ponens) while cancelling single-rule backward converse leakage. Third, \emph{symmetric port passivity} ($R_{\text{rule}} \succeq 0$) dissipates cyclic feedback energy on reciprocal rule pairs such as $\text{Cold}(x) \leftrightarrow \text{Kind}(x)$, preventing unbounded amplification across implicit solver steps.

\subsection{Passivity-Balanced Port-Hamiltonian Rule Diode}
At iteration $k$, let $\tilde{\mathbf{s}}^{(k)} = \operatorname{clamp}(\mathbf{z}^{(k)}, 0, s_{\max})$ denote bounded support and $s_{r,m}^{(k)} = \tilde{\mathbf{b}}_{r,m}^\top \tilde{\mathbf{s}}^{(k)}$ ($m \in \{1, 2\}$) denote normalized premise readings via soft-clamped readout vectors $\tilde{\mathbf{b}}_{r,m} = \operatorname{clamp}(4\|\mathbf{b}_{r,m}\|_2, 0, 1)\,\mathbf{b}_{r,m} \allowbreak / \allowbreak \max(\|\mathbf{b}_{r,m}\|_2, \varepsilon_0)$ ($\varepsilon_0 = 10^{-3}$). For two-premise conjunctive rules, the soft-minimum bottleneck weight $\theta_{r,2} = a_{r,2}\sigma(10(s_{r,1}^{(k)} - s_{r,2}^{(k)}))$ tracks the weaker premise direction to form unit-$\ell_2$ rule ports:
\begin{equation}
\label{eq:ports}
\tilde{\mathbf{a}}_r^{(k)} = (1 - \theta_{r,2})\mathbf{b}_{r,1} + \theta_{r,2}\mathbf{b}_{r,2}, \qquad
\mathbf{a}_r^{(k)} = \frac{\tilde{\mathbf{a}}_r^{(k)}}{\max(\|\tilde{\mathbf{a}}_r^{(k)}\|_2, \varepsilon_0)}, \qquad
\mathbf{c}_r = \frac{\mathbf{c}_{\text{raw},r}}{\max(\|\mathbf{c}_{\text{raw},r}\|_2, \varepsilon_0)}.
\end{equation}
Each grounded rule couples $\mathbf{a}_r^{(k)}$ and $\mathbf{c}_r$ directly with balanced port impedances $\alpha_r^{(k)} = \beta_r^{(k)} = \sqrt{w_r^{(k)} / 2}$, where the state-gated conductance $w_r^{(k)} = \bar{w}_r g_r^{(k)} \nu_r^{(k)} \ge 0$ combines port-norm validity $\nu_r^{(k)} = \operatorname{clamp}(4\|\tilde{\mathbf{a}}_r^{(k)}\|_2 \|\mathbf{c}_{\text{raw},r}\|_2, 0, 1) \in [0, 1]$ with two-premise conjunctive gating
\begin{equation*}
g_r^{(k)} = (1 - a_{r,2})\psi(s_{r,1}^{(k)}) + a_{r,2}\psi(s_{r,1}^{(k)})\psi(s_{r,2}^{(k)}),
\end{equation*}
using $\psi(s) = \sigma(\tau_{\text{gate}}(s - \theta_{\text{fire}}))$ in \textbf{TE-BPHD} (or threshold-free linear clamping $\psi(s) = \operatorname{clamp}(s, 0, 1)$ in \textbf{Threshold-Free TE-BPHD}; three-premise \textbf{Multi-Slot OT Grounder} rules incorporate $(\mathbf{b}_{r,3}, a_{r,3}, \theta_{r,3})$ as specified in Section~\ref{sec:frontend}).

Interconnecting skew-symmetric coupling $J(\mathbf{z}^{(k)}) = \sum_{r=1}^M \frac{w_r^{(k)}}{2}(\mathbf{c}_r \mathbf{a}_r^{(k)\top} - \mathbf{a}_r^{(k)} \mathbf{c}_r^\top)$ with rank-$1$ positive semi-definite rule-port dissipation $R_{\text{rule}}(\mathbf{z}^{(k)}) = \sum_{r=1}^M \frac{w_r^{(k)}}{2}(\mathbf{a}_r^{(k)} - \mathbf{c}_r)(\mathbf{a}_r^{(k)} - \mathbf{c}_r)^\top \succeq 0$ yields the passivity-balanced rule diode operator
\begin{equation}
\label{eq:rule_diode}
M_{\text{rule}}(\mathbf{z}^{(k)}) = J(\mathbf{z}^{(k)}) - R_{\text{rule}}(\mathbf{z}^{(k)}) = \sum_{r=1}^M \left[ w_r^{(k)} \mathbf{c}_r \mathbf{a}_r^{(k)\top} - \frac{w_r^{(k)}}{2}\big(\mathbf{a}_r^{(k)}\mathbf{a}_r^{(k)\top} + \mathbf{c}_r\mathbf{c}_r^\top\big) \right].
\end{equation}
The outer-product cross terms $-\frac{w_r^{(k)}}{2}\mathbf{a}_r^{(k)}\mathbf{c}_r^\top$ from $J$ and $+\frac{w_r^{(k)}}{2}\mathbf{a}_r^{(k)}\mathbf{c}_r^\top$ from $-R_{\text{rule}}$ cancel algebraically on every rule, doubling forward Modus Ponens conduction ($+w_r^{(k)}\mathbf{c}_r\mathbf{a}_r^{(k)\top}$) without adding an outer-product converse term ($0\cdot \mathbf{a}_r^{(k)}\mathbf{c}_r^\top$). Under learned soft unit ports where $t_r = \mathbf{a}_r^{(k)\top}\mathbf{c}_r$ need not be identically zero, the single-rule inner-product projection evaluates to $\mathbf{a}_r^{(k)\top} M_r \mathbf{c}_r = w_r^{(k)}(t_r^2 - t_r)$ (zero at $t_r = 0$); across $40{,}206$ active development rule ports ($3$ seeds), $t_r$ concentrates tightly at zero (median $0.0000$, $p_{99} = 0.0034$), yielding mean $w_r^{(k)}(t_r^2 - t_r) = -1.8 \times 10^{-4}$ with $99.8\%$ non-positive.

\subsection{Sylvester Passivity vs.\ Two-Cycle Resolvent Poles}
Notice in Eq.~\eqref{eq:rule_diode} that rule passivity introduces self-dissipation $-\frac{w_r}{2}\mathbf{a}_r\mathbf{a}_r^\top$ on the antecedent port, reducing isolated high-fan-out premise amplitude when a single fact feeds $k_{\text{fan}}$ rules simultaneously. One might ask whether removing antecedent self-dissipation---yielding the \textbf{Gain-Matched Unilateral Horn} operator $M_r^{\text{horn}} = w_r \mathbf{c}_r(\mathbf{a}_r - \frac{1}{2}\mathbf{c}_r)^\top$, which matches Eq.~\eqref{eq:rule_diode}'s exact single-rule forward equilibrium gain $w_r / (\gamma_0 + w_r/2)$ while preserving $100\%$ isolated fan-out amplitude---improves multi-hop reasoning. Proposition~\ref{prop:sylvester} proves why doing so breaks Euclidean passivity and creates a cyclic resolvent pole---an implicit-solver resonance where reciprocal rules amplify feedback unboundedly at a critical conductance.

\begin{proposition}[Sylvester $2\times 2$ Passivity Inequality and Two-Cycle Resolvent Pole]
\label{prop:sylvester}
Let $\mathbf{a}_r, \mathbf{c}_r \in \mathbb{R}^D$ satisfy $\|\mathbf{a}_r\|_2 = \|\mathbf{c}_r\|_2 = 1$ and $\mathbf{a}_r^\top \mathbf{c}_r = 0$, and consider any linear rule operator $M_r = J_r - R_r$ on $\mathcal{V}_r = \operatorname{span}\{\mathbf{a}_r, \mathbf{c}_r\}$ with skew-symmetric $J_r = -J_r^\top$ and symmetric $R_r = R_r^\top$ achieving forward conduction $\mathbf{c}_r^\top M_r \mathbf{a}_r = f_r > 0$ and zero single-rule converse coupling $\mathbf{a}_r^\top M_r \mathbf{c}_r = 0$. Then:
\begin{enumerate}
    \item \emph{(Minimal Passivity Cost):} $R_r \succeq 0$ holds if and only if $R_{a,a} \ge 0$, $R_{c,c} \ge 0$, and $R_{a,a} R_{c,c} \ge f_r^2 / 4$, where $R_{a,a} = \mathbf{a}_r^\top R_r \mathbf{a}_r$ and $R_{c,c} = \mathbf{c}_r^\top R_r \mathbf{c}_r$. Minimizing total trace dissipation $R_{a,a} + R_{c,c}$ at fixed forward gain $f_r = w_r$ uniquely yields the symmetric Port-Hamiltonian diode $R_{a,a} = R_{c,c} = w_r / 2$ in Eq.~\eqref{eq:rule_diode}.
    \item \emph{(Unilateral Horn Two-Cycle Pole):} Setting $R_{a,a} = 0$ and $R_{c,c} = w_r / 2$ ($M_r^{\text{horn}} = w_r \mathbf{c}_r(\mathbf{a}_r - \frac{1}{2}\mathbf{c}_r)^\top$) yields $\det(R_r|_{\mathcal{V}_r}) = -w_r^2 / 4 < 0$. On a reciprocal two-cycle of rules $(\mathbf{a}_1 = \mathbf{c}_2 = \mathbf{e}_1,\; \mathbf{c}_1 = \mathbf{a}_2 = \mathbf{e}_2,\; \mathbf{e}_1^\top \mathbf{e}_2 = 0)$ with equal conductance $w_1 = w_2 = w_r > 0$, $M_{\text{cyc}}^{\text{horn}} = M_1^{\text{horn}} + M_2^{\text{horn}}$ has a strictly positive eigenvalue $\lambda_+ = +w_r / 2 > 0$. Under implicit backward-Euler discretization with step size $\Delta t > 0$, isotropic metric $p > 0$, and leak $\gamma_0 > 0$ on un-clamped derived coordinates ($\mathbf{r}_{\text{src}} = \mathbf{0}$), the step amplification factor $\rho(w_r) = [1 + \Delta t\, p(\gamma_0 - w_r / 2)]^{-1}$ crosses a singular resolvent pole at $w_r = 2(\gamma_0 + (\Delta t\, p)^{-1})$, whereas Eq.~\eqref{eq:rule_diode} satisfies $M_{\text{cyc}}|_{\operatorname{span}\{\mathbf{e}_1,\mathbf{e}_2\}} \preceq 0$ and $|\rho_{\text{PH}}(w_r)| \le (1 + \Delta t\, p \gamma_0)^{-1} < 1$ for all $w_r \ge 0$.
\end{enumerate}
\end{proposition}

\subsection{Pointwise Discrete \texorpdfstring{$P$}{P}-Quadratic ISS and Rank-\texorpdfstring{$2M$}{2M} Woodbury Solver}
Let $P = \operatorname{diag}(\mathbf{p}) \succ 0$ denote a learnable positive diagonal metric ($p_i = \operatorname{softplus}(\tilde{p}_{\text{pos/neg}}) + 0.5$), $\gamma_0 \in (0.02, 0.06)$ denote background leak, and $\mathbf{r}_{\text{src}} = \mathbf{g}_{\text{src}} \odot (\mathbf{u}_{\text{fact}} \oslash (\mathbf{u}_{\text{fact}} + 0.05)) \in \mathbb{R}_{\ge 0}^{2N}$ denote active Norton source damping. Initialized at $\mathbf{z}^{(0)} = \mathbf{u}_{\text{fact}}$, each step $k \in \{0, \dots, K-1\}$ solves the semi-implicit backward-Euler system
\begin{equation}
\label{eq:implicit_step}
\Big[ I_{2N} + \Delta t \big( \gamma_0 I_{2N} + \operatorname{diag}(\mathbf{r}_{\text{src}}) - M_{\text{rule}}(\mathbf{z}^{(k)}) \big) P \Big] \mathbf{z}^{(k+1)} = \mathbf{z}^{(k)} + \Delta t \big( \mathbf{r}_{\text{src}} \odot P \mathbf{u}_{\text{fact}} \big).
\end{equation}

\begin{proposition}[Pointwise Discrete $P$-Quadratic ISS and Rank-$2M$ Woodbury Equivalence]
\label{prop:iss}
Suppose $\Delta t > 0$, $P = \operatorname{diag}(\mathbf{p}) \succ 0$, $\gamma_0 > 0$, $\mathbf{r}_{\text{src}} \in \mathbb{R}_{\ge 0}^{2N}$, $\mathbf{u}_{\text{fact}} \in \mathbb{R}_{\ge 0}^{2N}$, and state-dependent conductances satisfy $w_r(\mathbf{z}^{(k)}) \ge 0$ for all $r \in \{1, \dots, M\}$. Let $V_P(\mathbf{z}) = \frac{1}{2}\mathbf{z}^\top P \mathbf{z}$ and $\|\mathbf{v}\|_P = \sqrt{\mathbf{v}^\top P \mathbf{v}}$. Then:
\begin{enumerate}
    \item \emph{(Exact Pointwise Energy Identity and ISS Bound):} Every update $\mathbf{z}^{(k)} \to \mathbf{z}^{(k+1)}$ of Eq.~\eqref{eq:implicit_step} under Eq.~\eqref{eq:rule_diode} is nonsingular and satisfies the exact discrete Lyapunov identity
    \begin{equation}
    \label{eq:iss_balance}
    V_P(\mathbf{z}^{(k+1)}) - V_P(\mathbf{z}^{(k)}) + \frac{1}{2}\|\mathbf{z}^{(k+1)} - \mathbf{z}^{(k)}\|_P^2 + \Delta t\,\mathbf{y}^{(k+1)\top}\big(\gamma_0 I_{2N} + R_{\text{rule}}(\mathbf{z}^{(k)})\big)\mathbf{y}^{(k+1)} + \Delta t\sum_{i=1}^{2N} r_{\text{src},i} y_i^{(k+1)}\big(y_i^{(k+1)} - P_{ii} u_{\text{fact},i}\big) = 0,
    \end{equation}
    where $\mathbf{y}^{(k+1)} = P \mathbf{z}^{(k+1)} = \nabla V_P(\mathbf{z}^{(k+1)})$, yielding for bounded fixed fact clamping $\mathbf{u}_{\text{fact}}$ the uniform trajectory bound $\|\mathbf{z}^{(k)}\|_P \le \max\big(\|\mathbf{z}^{(0)}\|_P,\; \frac{\|\sqrt{\mathbf{r}_{\text{src}}} \odot P \mathbf{u}_{\text{fact}}\|_2}{2\sqrt{\gamma_0 p_{\min}}}\big)$ for all $k \ge 0$ ($p_{\min} = \min_i p_i$).
    \item \emph{(Rank-$2M$ Woodbury Equivalence):} Let $\mathbf{b}_k = \mathbf{z}^{(k)} + \Delta t(\mathbf{r}_{\text{src}} \odot P \mathbf{u}_{\text{fact}})$, $D_0 = I_{2N} + \Delta t(\gamma_0 I_{2N} + \operatorname{diag}(\mathbf{r}_{\text{src}}))P$, and $\mathbf{y}_{0,k} = D_0^{-1}\mathbf{b}_k$. Factoring $M_{\text{rule}}(\mathbf{z}^{(k)}) = U_k V_k^\top$ with $U_k = [2B_k - A_k,\; -B_k] \in \mathbb{R}^{2N \times 2M}$ and $V_k = [A_k,\; B_k] \in \mathbb{R}^{2N \times 2M}$, where $A_k = [\sqrt{w_1^{(k)}/2}\,\mathbf{a}_1^{(k)}, \dots, \sqrt{w_M^{(k)}/2}\,\mathbf{a}_M^{(k)}]$ and $B_k = [\sqrt{w_1^{(k)}/2}\,\mathbf{c}_1, \dots, \sqrt{w_M^{(k)}/2}\,\mathbf{c}_M]$, Eq.~\eqref{eq:implicit_step} reduces identically via the Sherman--Morrison--Woodbury formula~\cite{woodbury1950inverting} to a $2M \times 2M$ Capacitance solve $S_k \boldsymbol{\xi}_k = V_k^\top P \mathbf{y}_{0,k}$ ($192 \times 192$ for $M = 96$ grounded rule ports, $2M \ll D = 1{,}440$) with $S_k = I_{2M} - \Delta t V_k^\top P D_0^{-1} U_k$, recovering $\mathbf{z}^{(k+1)} = \mathbf{y}_{0,k} + \Delta t\, D_0^{-1} U_k \boldsymbol{\xi}_k$ in $\mathcal{O}(DM^2 + M^3)$ rather than $\mathcal{O}(D^3)$ operations.
\end{enumerate}
\end{proposition}

Complete algebraic proofs and machine-precision numerical checks (\texttt{float64} Woodbury parity $1.33 \times 10^{-15}$; pointwise energy residual $8.53 \times 10^{-14}$) are provided in Supplementary Section~S1. Because $M_{\text{rule}}(\mathbf{z}^{(k)})$ depends on $\mathbf{z}^{(k)}$ through premise gating, Proposition~\ref{prop:iss} establishes trajectory boundedness (ISS) rather than global contraction to a unique nonlinear fixed point or un-clamped non-negativity prior to readout clamping.

\subsection{Cross-Plugged Deductive Backends: Saturating Fuzzy Logic vs.\ Gated Graph Recurrence}
Proposition~\ref{prop:sylvester} applies specifically to orthonormal single-rule dissipation blocks within unbounded additive linear implicit solvers. To test how Port-Hamiltonian passivity compares against nonlinear saturation on a compact hypercube $[0, 1]^{2N}$, we cross-plug two published operator families onto the identical clause-grounder architecture:
\begin{itemize}
    \item \textbf{Grounded Product-T-Norm LTN}~\cite{badreddine2022logic,vankrieken2022analyzing}: Evaluates Reichenbach Product-T-Norm conjunction $t_r^{(k)} \in [0, 1]$ over clamped premise readings $s_{r,m}^{(k)} \in [0, 1]$, scales rule activation $g_r^{(k)} = \operatorname{clamp}(\tanh(\kappa_{\text{rule}}\bar{w}_r) t_r^{(k)} \nu_r^{(k)}, 0, 0.99)$, and aggregates base facts $u_{\text{root},i} = \operatorname{clamp}(\kappa_{\text{fact}} u_{\text{fact},i}, 0, 0.99)$ with rule consequents via the algebraic Probabilistic-Sum S-norm (Noisy-OR):
    \begin{equation}
    \label{eq:ltn}
    z_i^{(k+1)} = 1 - (1 - u_{\text{root},i}) \prod_{r=1}^M \big(1 - g_r^{(k)} c_{r,i}\big) \in [0, 1), \qquad i \in \{1, \dots, 2N\},
    \end{equation}
    requiring $\mathcal{O}(DM)$ matrix-vector products and zero matrix inversions while structurally bounding cyclic trajectories inside $[0, 1)^{2N}$.
    \item \textbf{Grounded Port-GRU}: Replaces structured deductive algebra with a $10$-feature rule-port multilayer perceptron ($144$ hidden units, RMS normalization, $\tanh$ bounding) and per-atom bilateral Gated Recurrent Unit (GRU) updates over $K = 6$ steps, isolating structured logical aggregation from generic gated graph message passing. Preliminary development-only screenings of discarded exploratory operators (asymmetric high-impedance diodes, tropical $\ell_\infty$ monotone updates, and rule-undamped skew flow) are documented in Supplementary Section~S4 and Supplementary Table~S5.
\end{itemize}

\section{Experimental Protocol and Baselines}
\label{sec:protocol}

\textbf{Theory-Disjoint Benchmark and Selection Firewall.} All experiments use official \texttt{ProofWriter V2020.12.3} OWA splits~\cite{tafjord2021proofwriter,proofwriter2020dataset} (\url{https://allenai.org/data/proofwriter}) with zero overlap of source-theory identifiers across four partitions: training ($10{,}477$ \texttt{depth-5} queries; plus $19{,}614$ \texttt{ParaRules} queries in Regime~II), calibration ($1{,}538$ \texttt{depth-5} queries), development ($1{,}876$ \texttt{depth-5} standard, $600$ \texttt{depth-5} counterfactual across $200$ triplets, $600$ unseen-connective paraphrase counterfactual, and $948$ \texttt{ParaRules} counterfactual across $237$ families), and the single-pass locked test suite ($17{,}036$ queries: $1{,}958$ \texttt{depth-5} standard, $600$ \texttt{depth-5} All-3 counterfactual across $200$ triplets, $8{,}008$ \texttt{ParaRules} standard, $1{,}200$ \texttt{ParaRules} All-4 counterfactual across $300$ families, and $5{,}270$ zero-shot OOD queries across six \texttt{birds-electricity} domains spanning $C = 140$ theory clusters). Each \texttt{depth-5} All-3 triplet pairs one \texttt{True}, one \texttt{False}, and one \texttt{Unknown} context ($200{:}200{:}200$ balanced across six mechanisms---converse probes, upstream or local argument/entity swaps, upstream root drops, conjunctive half-premise drops, and upstream premise drops---with $100.0\%$ forward signed Horn verification). All architectures, five-group \texttt{AdamW} schedules, and test families were locked via calibration-only selection prior to single-pass locked test evaluation across seeds $\{42, 123, 910\}$ (Supplementary Section~S3 and Supplementary Table~S4).

\textbf{Matched Neural, Port-Coupling Ablation, and Pretrained LLM Baselines.} Under Regime~I (\texttt{depth-5}) and Regime~II (joint \texttt{depth-5} + \texttt{ParaRules}), we compare against converged \textbf{Seq2Seq T5-small} ($60.51\text{M}$, $24$ epochs Regime~I / $16$ epochs Regime~II), \textbf{Flan-T5-small}~\cite{chung2024scaling} ($76.96\text{M}$, $16$ epochs), and latent \textbf{Recurrent-Depth T5}~\cite{geiping2025scaling} ($35.99\text{M}$, $K=6$ iterations, $12$ epochs in Regime~I and Regime~II), all with $512$-token windows ($0\%$ truncation). We also evaluate \textbf{Out-Degree-Scaled TE-BPHD}, combining Gram--Schmidt consequent orthogonalization ($\mathbf{c}_r \leftarrow \mathbf{c}_{\text{raw},r} - (\mathbf{a}_r^{(k)\top} \mathbf{c}_{\text{raw},r})\mathbf{a}_r^{(k)}$) with antecedent out-degree scaling. On a stratified $600$-query confirmation slice ($120$ queries/suite, intact families, balanced across six OOD domains), we evaluate \textbf{Qwen3-1.7B}~\cite{yang2025qwen3} ($1.72\text{B}$ parameters) under $0$-shot and $3$-shot OWA prompting across \texttt{Direct} decoding and greedy chain-of-thought capped at $384$ or $2{,}048$ reasoning tokens (with deterministic single-step post-cap $\{\texttt{A}, \texttt{B}, \texttt{C}\}$ logit fallback when a trace hits the cap without closing).

\textbf{Capacity-Normalized Metrics, Hardware Timing, and Statistical Testing.} Counterfactual faithfulness requires exact familywise consistency across all $3$ twin queries on \texttt{depth-5} (\texttt{All-3 CF}) and all $4$ twin queries on \texttt{ParaRules} (\texttt{All-4 CF}). Because \texttt{ElectricityRB4} comprises $4{,}224$ of $5{,}270$ OOD queries ($80.15\%$) whereas \texttt{BirdsVar1/2} contain $40$ queries ($1$ cluster) each, we co-report query-weighted OOD accuracy, cluster-equal-weight differences ($\Delta_{\text{clust}}$, $C=140$), cluster medians, and six-domain macro-means. Throughput ($\text{ms/query}$, including T5 encoding) is timed at \texttt{batch\_size=16} ($2$ warmup + $3$ CUDA-synchronized passes over $120$ \texttt{depth-5} standard queries on one $24\,\text{GB}$ \texttt{NVIDIA L4} GPU, reporting amortized batch-$16$ throughput rather than batch-$1$ online latency; greedy accuracy shards for $3$-shot $2{,}048$-token \texttt{ParaRules} and \texttt{birds-electricity} ran on \texttt{NVIDIA A100-SXM4-80GB}). Correctness scores are averaged across seeds $\{42, 123, 910\}$ per source-theory cluster prior to $B = 10{,}000$ paired cluster-bootstrap CIs and two-sided Monte Carlo paired theory-cluster sign-flip randomization tests ($B = 10{,}000$ flips, RNG seed $20261103$, $p = (1 + k)/(B + 1) \ge 1/10{,}001$)~\cite{phipson2010permutation}---conditioning cluster CIs/$p$-values on the $3$-seed ensemble while cross-seed $\pm\,\text{SD}$ captures initialization variance---with Holm--Bonferroni adjustment across the $50$-test primary ($p_{\text{Holm}}$) and $75$-test omnibus ($p_{\text{Holm75}}$) families (Supplementary Section~S2 and Supplementary Table~S2).

\section{Results and Mechanistic Analysis}
\label{sec:results}

\subsection{Counterfactual Deduction, Cross-Domain Transfer, and the Throughput Frontier}

\begin{table*}[t]
\centering
\footnotesize
\setlength{\tabcolsep}{2.5pt}
\caption{\textbf{Primary Regime I Locked Test Evaluation ($17{,}036$ Queries) and Hardware-Locked Batch-$16$ NVIDIA L4 Throughput.}}
\label{tab:locked_test_main}
\begin{tabular}{@{}lccccccccc@{}}
\toprule
\textbf{Model} & \textbf{Params} & \textbf{D5 Std $\uparrow$} & \textbf{D5 CF $\uparrow$} & \textbf{PR Std $\uparrow$} & \textbf{PR CF $\uparrow$} & \textbf{OOD $\uparrow$} & \textbf{$\Delta_{\text{clust}}$ $\uparrow$} & \textbf{Macro $\uparrow$} & \textbf{Lat. $\downarrow$} \\
 & \textbf{(M)} & \textbf{($1{,}958$)} & \textbf{($600$)} & \textbf{($8{,}008$)} & \textbf{($1{,}200$)} & \textbf{($5{,}270$)} & \textbf{(pp)} & \textbf{(6 dom.)} & \textbf{(ms/q)} \\
\midrule
\textbf{TE-BPHD (Proposed)} & $37.85$ & $\mathbf{86.96 \pm 1.90}$ & $\mathbf{79.17 \pm 5.48}$ & $51.16 \pm 1.41$ & $1.33 \pm 1.76$ & $\mathbf{94.95 \pm 1.23}$ & --- & $\mathbf{92.54 \pm 3.38}$ & $9.44$ \\
Out-Deg.\ TE-BPHD & $37.85$ & $81.07 \pm 4.46$ & $78.83 \pm 3.06$ & $50.96 \pm 1.24$ & $1.00 \pm 0.88$ & $94.49 \pm 3.14$ & ${-}0.28$ & $90.91 \pm 4.48$ & $9.54$ \\
Thresh.-Free TE-BPHD & $37.85$ & $84.97 \pm 2.92$ & $77.33 \pm 4.91$ & $51.25 \pm 1.19$ & $0.89 \pm 1.02$ & $92.72 \pm 5.93$ & ${-}2.66$ & $87.79 \pm 1.51$ & $9.31$ \\
Prod.-T-Norm LTN & $37.85$ & $84.03 \pm 0.77$ & $73.67 \pm 8.81$ & $51.72 \pm 1.77$ & $0.78 \pm 1.07$ & $94.50 \pm 0.87$ & ${-}1.53$ & $86.08 \pm 1.22$ & $8.15$ \\
Port-GRU ($K{=}6$) & $37.85$ & $81.99 \pm 5.88$ & $72.50 \pm 2.29$ & $51.09 \pm 1.48$ & $0.89 \pm 0.84$ & $79.98 \pm 13.27$ & ${-}15.22$ & $82.22 \pm 5.02$ & $8.99$ \\
Unilat.\ Horn ($R_{a,a}{=}0$) & $37.85$ & $61.13 \pm 7.87$ & $29.17 \pm 30.03$ & $50.75 \pm 0.62$ & $0.00 \pm 0.00$ & $88.21 \pm 3.35$ & ${-}7.26$ & $81.67 \pm 1.56$ & $8.76$ \\
\midrule
Seq2Seq T5 ($24\,$ep) & $60.51$ & $74.23 \pm 0.45$ & $47.67 \pm 0.76$ & $55.96 \pm 1.26$ & $6.44 \pm 3.20$ & $60.28 \pm 2.00$ & ${-}34.86$ & $64.91 \pm 3.44$ & $2.01$ \\
Flan-T5 ($16\,$ep) & $76.96$ & $65.06 \pm 0.97$ & $40.50 \pm 1.32$ & $\mathbf{59.33 \pm 3.62}$ & $12.56 \pm 6.53$ & $54.73 \pm 9.80$ & ${-}40.70$ & $56.41 \pm 7.66$ & $2.67$ \\
Rec.-Depth T5 ($K{=}6$) & $35.99$ & $78.86 \pm 2.22$ & $55.33 \pm 8.39$ & $58.71 \pm 1.66$ & $\mathbf{13.33 \pm 5.55}$ & $66.36 \pm 12.07$ & ${-}28.90$ & $69.72 \pm 7.14$ & $3.08$ \\
\midrule
Multi-Slot OT (s42)$^{\dagger}$ & $38.57$ & $71.65$ & $78.00$ & $57.60$ & $10.67$ & $87.59$ & ${-}8.00$ & $80.08$ & $14.86$ \\
\bottomrule
\end{tabular}
\vspace{2pt}\par\noindent\parbox{\textwidth}{\scriptsize \textit{Notes:} Trained solely on \texttt{depth-5} ($10{,}477$ train, $1{,}538$ calibration) and evaluated once across three matched seeds ($\{42, 123, 910\}$, accuracy \% $\text{mean} \pm \text{SD}$; query counts in parentheses). \textbf{D5}: \texttt{depth-5}; \textbf{PR}: \texttt{ParaRules}. On six-domain \texttt{birds-electricity} ($N=5{,}270$ across $C=140$ theory clusters, where \texttt{ElectricityRB4} has $4{,}224$ queries and \texttt{BirdsVar1/2} have $40$ queries in $1$ cluster each), we co-report query-weighted accuracy (\textbf{OOD}), paired cluster-equal-weight difference $\Delta_{\text{clust}} = \text{Model} - \text{TE-BPHD}$ across $C=140$ clusters (negative values indicate lower accuracy than \textbf{TE-BPHD}), and unweighted six-domain macro-mean (\textbf{Macro}). \textbf{Lat.} reports amortized end-to-end \texttt{batch=16} throughput ($\text{ms/q}$, $2$ warmup batches followed by $3$ synchronized timed passes over $120$ \texttt{depth-5} standard queries on a single $24\,\text{GB}$ \texttt{NVIDIA L4} GPU, rather than batch-$1$ single-request online latency); shared-theory deduplicated \texttt{batch=16} throughput is $7.65\,\text{ms/q}$ for \textbf{TE-BPHD}, $7.99\,\text{ms/q}$ for Out-Deg.\ TE-BPHD, $7.61\,\text{ms/q}$ for Thresh.-Free TE-BPHD, $6.77\,\text{ms/q}$ for Prod.-T-Norm LTN, $7.36\,\text{ms/q}$ for Port-GRU, $7.08\,\text{ms/q}$ for Unilat.\ Horn, and $11.86\,\text{ms/q}$ for Multi-Slot OT. $^{\dagger}$Evaluated on seed $42$ after development freeze and paired $1{:}1$ against seed-$42$ \textbf{TE-BPHD} ($84.78\%$ D5 Std, $73.00\%$ D5 CF, $0.67\%$ PR CF, $93.72\%$ OOD).}
\end{table*}

Table~\ref{tab:locked_test_main} and Figure~\ref{fig:locked_test_pareto} summarize single-pass locked test performance under Regime~I (\texttt{depth-5} training). Explicit bilateral clause grounding paired with passivity-balanced Port-Hamiltonian deliberation (\textbf{TE-BPHD}, $37.85\text{M}$ parameters) attains $86.96 \pm 1.90\%$ standard \texttt{depth-5} accuracy and $79.17 \pm 5.48\%$ All-3 counterfactual consistency. Relative to converged token-level transformers trained on the identical split, \textbf{TE-BPHD} improves \texttt{depth-5} counterfactual consistency by $+31.50$ percentage points over \textbf{Seq2Seq T5-small} ($95\%$ CI $[+23.69, +39.29]$, $p_{\text{Holm}} = 0.0050$), $+38.67$ percentage points over \textbf{Flan-T5-small} ($95\%$ CI $[+30.71, +46.60]$, $p_{\text{Holm}} = 0.0050$), and $+23.83$ percentage points over $K=6$ \textbf{Recurrent-Depth T5} ($95\%$ CI $[+16.58, +31.06]$, $p_{\text{Holm}} = 0.0050$). Compared against composite \textbf{Out-Degree-Scaled TE-BPHD} (Gram--Schmidt consequent orthogonalization plus antecedent out-degree scaling), unscaled passivity-balanced coupling raises standard \texttt{depth-5} test accuracy by $+5.89$ percentage points ($86.96\%$ vs.\ $81.07\%$, $95\%$ CI $[+4.85, +6.95]$, $p_{\text{Holm}} = 0.0050$) while retaining full counterfactual consistency ($79.17\%$ vs.\ $78.83\%$, $p_{\text{Holm}} = 1.0000$).

\begin{figure*}[t]
\centering
\includegraphics[width=\textwidth]{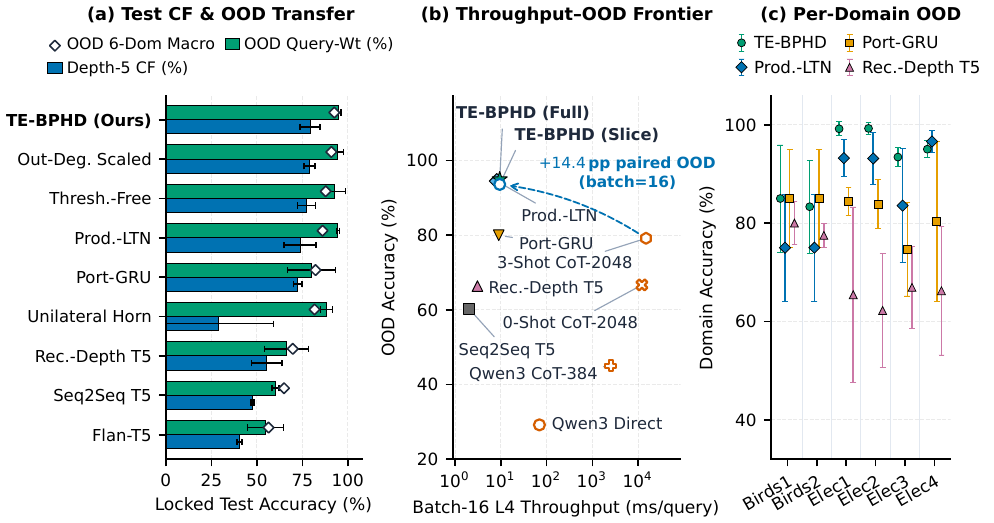}
\caption{\textbf{Primary Locked Test Performance, Throughput--Accuracy Frontier, and Per-Domain OOD Transfer.} \textbf{(a)} Three-seed locked test \texttt{depth-5} All-3 counterfactual consistency and six-domain \texttt{birds-electricity} OOD accuracy (bars: query-weighted $\text{mean} \pm \text{SD}$; diamonds: unweighted six-domain macro-mean). \textbf{(b)} Amortized \texttt{batch=16} \texttt{NVIDIA L4} throughput ($\text{ms/query}$, log scale) vs.\ OOD accuracy across full $5{,}270$-query models (filled markers) and paired $120$-query OOD slice evaluations (open markers; dashed arrow links \textbf{TE-BPHD} slice $93.61\%$ vs.\ $3$-shot \textbf{Qwen3-1.7B} $79.17\%$, $+14.44\,$pp). \textbf{(c)} Grouped discrete per-domain zero-shot accuracy ($\text{mean} \pm \text{SD}$, no connecting lines) across all six \texttt{birds-electricity} domains (Supplementary Section~S2 and Supplementary Table~S1).}
\label{fig:locked_test_pareto}
\end{figure*}

Zero-shot transfer to the six-domain \texttt{birds-electricity} suite ($N = 5{,}270$, $C = 140$ theory clusters) reveals an even sharper separation between token-level latent recurrence and explicit bilateral port grounding. \textbf{TE-BPHD} achieves $94.95 \pm 1.23\%$ query-weighted OOD accuracy and a $92.54 \pm 3.38\%$ unweighted six-domain macro-mean, outperforming \textbf{Seq2Seq T5-small}, \textbf{Flan-T5-small}, and \textbf{Recurrent-Depth T5} by $+28.58$ to $+40.22$ percentage points on query-weighted accuracy, corresponding to baseline deficits of $-28.90$ to $-40.70$ percentage points on cluster-equal-weight differences ($\Delta_{\text{clust}} = \text{Model} - \text{TE-BPHD}$; cluster median differences of $28.79$ to $40.15$ percentage points, all $p_{\text{Holm}} = 0.0050$). Across all six individual OOD domains (Figure~\ref{fig:locked_test_pareto}c, Supplementary Section~S2, and Supplementary Table~S1), \textbf{TE-BPHD} maintains $83.33\%\text{--}85.00\%$ on biological taxonomy rules (\texttt{BirdsVar1/2}) and $93.43\%\text{--}99.26\%$ across circuit troubleshooting domains (\texttt{ElectricityRB1--4}). Moreover, sharing the identical clause-grounder architecture trained end-to-end, \textbf{TE-BPHD} exceeds generic gated graph message passing (\textbf{Grounded Port-GRU}, $79.98 \pm 13.27\%$ query-weighted, $82.22 \pm 5.02\%$ macro-mean) by $+14.97$ percentage points ($95\%$ CI $[+14.05, +15.87]$, $\Delta_{\text{clust}} = -15.22$ for \textbf{Grounded Port-GRU} vs.\ \textbf{TE-BPHD}, $p_{\text{Holm}} = 0.0050$), confirming that structured deductive support aggregation generalizes more reliably across unseen domain vocabularies than unconstrained port-MLP recurrence.

\begin{table}[!ht]
\centering
\footnotesize
\setlength{\tabcolsep}{2.6pt}
\caption{\textbf{Paired $600$-Query Confirmation Slice Comparing TE-BPHD and Pretrained Qwen3-1.7B Across Reasoning Budgets.}}
\label{tab:qwen_slice}
\begin{tabular}{@{}lcccccccc@{}}
\toprule
\textbf{Model \& Decoding Regime} & \textbf{D5 Std $\uparrow$} & \textbf{D5 CF $\uparrow$} & \textbf{PR Std $\uparrow$} & \textbf{PR CF $\uparrow$} & \textbf{OOD Slice $\uparrow$} & \textbf{Close $\uparrow$} & \textbf{Latency $\downarrow$} & \textbf{$p_{\text{Holm}}$} \\
 & \textbf{($120$)} & \textbf{($120$)} & \textbf{($120$)} & \textbf{($120$)} & \textbf{($120$)} & \textbf{(\%)} & \textbf{(ms/q)} & \textbf{(OOD)} \\
\midrule
\textbf{TE-BPHD (Proposed)} & $\mathbf{87.50 \pm 2.50}$ & $\mathbf{70.83 \pm 3.82}$ & $38.06 \pm 2.40$ & $0.00 \pm 0.00$ & $\mathbf{93.61 \pm 4.59}$ & --- & $\mathbf{9.44}$ & --- \\
\midrule
Qwen3-1.7B 0-Shot Direct & $45.83$ & $0.00$ & $44.17$ & $0.00$ & $28.33$ & --- & $33.6$ & $0.0050$ \\
Qwen3-1.7B 0-Shot CoT ($384$ cap) & $45.83$ & $5.00$ & $46.67$ & $0.00$ & $35.83$ & $0.00$ & $1660.4$ & $1/10{,}001$$^{*}$ \\
Qwen3-1.7B 0-Shot CoT ($2{,}048$ cap) & $63.33$ & $22.50$ & $63.33$ & $\mathbf{30.00}$ & $66.67$ & $64.17$ & $12117.3$ & $0.0050$ \\
Qwen3-1.7B 3-Shot Direct & $40.00$ & $2.50$ & $45.00$ & $3.33$ & $29.17$ & --- & $69.6$ & $0.0050$ \\
Qwen3-1.7B 3-Shot CoT ($384$ cap) & $55.83$ & $5.00$ & $58.33$ & $16.67$ & $45.00$ & $0.83$ & $2533.2$ & $1/10{,}001$$^{*}$ \\
Qwen3-1.7B 3-Shot CoT ($2{,}048$ cap)$^{\ddagger}$ & $70.00$ & $50.00$ & $\mathbf{65.83}$ & $16.67$ & $79.17$ & $75.83$ & $14878.4$ & $0.0136$ \\
\bottomrule
\end{tabular}
\vspace{2pt}\par\noindent\parbox{\textwidth}{\scriptsize \textit{Notes:} Evaluated on stratified $120$-query slices per suite (accuracy in \%; $40$ counterfactual triplets on \texttt{depth-5}, $30$ quadruple families on \texttt{ParaRules}, $20$ queries per domain across all six \texttt{birds-electricity} domains). \textbf{D5}: \texttt{depth-5}; \textbf{PR}: \texttt{ParaRules}; \textbf{Close}: percentage of \texttt{depth-5} counterfactual reasoning traces naturally closing their reasoning block prior to the token cap ($35.83\%$ in $0$-shot and $24.17\%$ in $3$-shot reach the $2{,}048$-token cap without natural closure and are resolved via deterministic post-cap single-step $\{\texttt{A},\texttt{B},\texttt{C}\} $ logit selection). $p_{\text{Holm}}$ reports $50$-test Holm--Bonferroni adjusted $p$-values on the paired $120$-query OOD slice. $^{*}$Raw Monte Carlo paired theory-cluster sign-flip randomization $p$-value ($B=10{,}000$ random sign flips, RNG seed $20261103$, $p=(1+k)/(B+1) \ge 1/10{,}001$) for secondary $384$-token budget controls. $^{\ddagger}$All \textbf{Latency} entries report amortized \texttt{batch=16} throughput on the same $24\,\text{GB}$ \texttt{NVIDIA L4} GPU ($2$ warmup batches + $3$ synchronized timed passes over $120$ \texttt{depth-5} standard queries; $799.0\,\text{MB}$ peak VRAM for \textbf{TE-BPHD} vs.\ $8{,}507.7\,\text{MB}$ for $2{,}048$-token Qwen3-1.7B); deterministic zero-temperature greedy accuracy shards for $3$-Shot OWA CoT ($2{,}048$-tok cap) on \texttt{ParaRules} and \texttt{birds-electricity} ran on \texttt{NVIDIA A100-SXM4-80GB}.}
\end{table}

Table~\ref{tab:qwen_slice} evaluates pretrained \textbf{Qwen3-1.7B} ($1.72\text{B}$ parameters) across explicit reasoning-token caps on the paired $600$-query confirmation slice ($120$ queries per suite). Under a $384$-token reasoning cap, at most $0.83\%$ of \texttt{depth-5} counterfactual reasoning traces close naturally, restricting $0$-shot and $3$-shot counterfactual consistency to $5.00\%$. Raising the cap to $2{,}048$ reasoning tokens substantially reduces---though does not completely eliminate---reasoning-budget truncation: natural trace closure rises to $64.17\%$ ($0$-shot, with $35.83\%$ reaching the cap) and $75.83\%$ ($3$-shot, with $24.17\%$ reaching the cap and resolved via deterministic post-cap single-step $\{\texttt{A},\texttt{B},\texttt{C}\}$ logit selection), elevating $3$-shot OWA \textbf{Qwen3-1.7B} to $50.00\%$ \texttt{depth-5} counterfactual consistency and $79.17\%$ OOD slice accuracy ($\sim 1{,}353$ generated tokens and $14{,}878.4\,\text{ms/query}$ on \texttt{NVIDIA L4}). Compared against higher-budget ($2{,}048$-token cap) greedy decoding on the identical $120$-query slices, the $37.85\text{M}$-parameter \textbf{TE-BPHD} model ($93.61 \pm 4.59\%$ OOD slice, $70.83 \pm 3.82\%$ \texttt{depth-5} CF slice) outperforms $0$-shot \textbf{Qwen3-1.7B} on both \texttt{depth-5} counterfactuals ($+48.33$ percentage points, $p_{\text{Holm}} = 0.0050$) and OOD ($+26.94$ percentage points, $p_{\text{Holm}} = 0.0050$), and surpasses $3$-shot \textbf{Qwen3-1.7B} by $+14.44$ percentage points on the paired $120$-query six-domain OOD slice ($95\%$ CI $[+5.98, +26.52]$, raw $p = 0.0008$, $p_{\text{Holm}} = 0.0136$) while showing a $+20.83$ percentage point nominal advantage on the $40$-triplet \texttt{depth-5} counterfactual slice ($95\%$ CI $[+5.68, +36.59]$, raw $p = 0.0216$, not familywise-significant after $50$-test correction at $p_{\text{Holm}} = 0.2592$). Hardware-locked \texttt{batch=16} \texttt{NVIDIA L4} throughput profiling places \textbf{TE-BPHD} at $9.44\,\text{ms/query}$ unamortized ($7.65\,\text{ms/query}$ with shared-theory deduplication, $799\,\text{MB}$ peak VRAM vs.\ $14{,}878.4\,\text{ms/query}$ and $8{,}507.7\,\text{MB}$ for $3$-shot $2{,}048$-token \textbf{Qwen3-1.7B}; Figure~\ref{fig:locked_test_pareto}b).

\subsection{Scoped Dynamical Dichotomy: Sylvester Passivity vs.\ Bounded Saturation}

\begin{table*}[t]
\centering
\footnotesize
\setlength{\tabcolsep}{2.5pt}
\caption{\textbf{Development-Split $2\times 3$ Grounder--Operator Factorial, Cross-Plugged Backends, Knockouts, and Mechanism Probes.}}
\label{tab:dev_factorial}
\begin{tabular}{@{}lcccccccc@{}}
\toprule
\textbf{Configuration} & \textbf{D5 Std $\uparrow$} & \textbf{D5 CF $\uparrow$} & \textbf{Unseen CF $\uparrow$} & \textbf{PR CF $\uparrow$} & \textbf{$k_{\text{fan}}{=}8$ Ret.$^{*}$} & \textbf{Conv. $\uparrow$} & \textbf{$K{=}24$ $\|\mathbf{z}\|_2$ $\downarrow$} & \textbf{Spec.\ Gap $\uparrow$} \\
 & \textbf{($1{,}876$)} & \textbf{($600$)} & \textbf{($600$)} & \textbf{($948$)} & \textbf{(\%)} & \textbf{(\%)} & \textbf{(max)} & \textbf{(pp)} \\
\midrule
\textbf{TE-BPHD (Proposed)} & $\mathbf{87.94 \pm 1.32}$ & $\mathbf{90.67 \pm 3.82}$ & $66.00 \pm 2.78$ & $1.41 \pm 1.76$ & $15.93$ & $98.96$ & $4.49$ & ${-}9.67$ \\
Out-Deg.\ TE-BPHD & $82.78 \pm 4.34$ & $77.50 \pm 22.61$ & $58.67 \pm 18.24$ & $0.28 \pm 0.48$ & $52.08$ & $91.67$ & $4.13$ & ${-}11.05$ \\
Unilat.\ Horn ($R_{a,a}{=}0$) & $63.41 \pm 8.28$ & $11.00 \pm 17.35$ & $8.50 \pm 12.58$ & $0.00 \pm 0.00$ & $100.00$ & $89.06$ & Overflow ($\infty$) & ${-}10.87$ \\
\midrule
5-Slot OT + Out-Deg.$^{\dagger}$ & $47.17$ & $37.50$ & $25.50$ & $9.28$ & $52.10$ & $79.69$ & $3.34$ & ${+}50.27$ \\
5-Slot OT + TE-BPHD$^{\dagger}$ & $69.62$ & $45.00$ & $36.00$ & $9.28$ & $15.14$ & $70.31$ & $4.65$ & ${+}66.66$ \\
5-Slot OT + Unilat.\ Horn$^{\dagger}$ & $50.43$ & $6.00$ & $1.50$ & $1.69$ & $100.00$ & $90.62$ & $3.42\times 10^{7}$ & ${+}46.64$ \\
\midrule
Thresh.-Free TE-BPHD & $85.04 \pm 1.40$ & $84.67 \pm 13.25$ & $71.67 \pm 13.53$ & $0.70 \pm 0.49$ & $23.20$ & $95.84$ & $5.07$ & ${-}9.67$ \\
Product-T-Norm LTN & $84.84 \pm 1.08$ & $77.83 \pm 22.81$ & $62.33 \pm 14.94$ & $0.56 \pm 0.65$ & $100.00$ & $94.79$ & $5.18$ & ${-}9.67$ \\
Port-GRU ($K{=}6$) & $82.53 \pm 5.21$ & $50.50 \pm 5.27$ & $41.00 \pm 6.06$ & $1.97 \pm 1.36$ & --- & --- & --- & --- \\
Zero-Step KO ($K{=}0$) & $57.78 \pm 0.19$ & $0.17 \pm 0.29$ & $0.17 \pm 0.29$ & $0.00 \pm 0.00$ & --- & --- & --- & --- \\
Null-Rule KO ($M_{\text{rule}}{=}0$) & $57.78 \pm 0.19$ & $0.17 \pm 0.29$ & $0.17 \pm 0.29$ & $0.00 \pm 0.00$ & --- & --- & --- & --- \\
\midrule
Seq2Seq T5 ($24\,$ep) & $76.21 \pm 0.06$ & $25.17 \pm 1.76$ & $15.33 \pm 2.25$ & $6.47 \pm 3.22$ & --- & --- & --- & --- \\
Flan-T5 ($16\,$ep) & $66.47 \pm 2.25$ & $18.17 \pm 2.47$ & $12.17 \pm 0.76$ & $8.30 \pm 5.73$ & --- & --- & --- & --- \\
Rec.-Depth T5 ($K{=}6$) & $81.02 \pm 3.00$ & $33.50 \pm 6.61$ & $18.17 \pm 2.02$ & $15.89 \pm 7.50$ & --- & --- & --- & --- \\
\bottomrule
\end{tabular}
\vspace{2pt}\par\noindent\parbox{\textwidth}{\scriptsize \textit{Notes:} Evaluated prior to opening the locked test suite across three matched seeds ($\{42, 123, 910\}$, accuracy \% $\text{mean} \pm \text{SD}$). \textbf{D5}: \texttt{depth-5} ($200$ triplets); \textbf{Unseen CF}: held-out unseen-connective paraphrase counterfactuals ($200$ triplets); \textbf{PR CF}: \texttt{ParaRules} All-4 counterfactuals ($237$ families). Mechanism probes report isolated premise retention at rule fan-out $k_{\text{fan}}=8$, converse \texttt{Unknown} recall, maximum Euclidean norm at $K=24$, and the content-minus-unseen-connective token specificity gap (\textbf{Spec.\ Gap}, in percentage points; $\pm 1.94\,\text{pp}$ SD across three seeds for \textbf{TE-BPHD}). $^{*}$Eliminating antecedent self-dissipation ($R_{a,a}=0$) retains $100\%$ isolated $k_{\text{fan}}=8$ amplitude at the cost of cyclic passivity, inducing two-cycle divergence ($\lambda_+ = {+}w_r/2 > 0$) and $K=24$ overflow. $^{\dagger}$Evaluated on seed $42$ during multi-slot screening.}
\end{table*}

Table~\ref{tab:dev_factorial} and Figure~\ref{fig:passivity_wavefront} test Proposition~\ref{prop:sylvester} and Proposition~\ref{prop:iss} through controlled operator interventions. On isolated synthetic fan-out probes ($k_{\text{fan}} \in \{1, 2, 4, 8\}$ rules sharing a single root premise, Figure~\ref{fig:passivity_wavefront}a), \textbf{Gain-Matched Unilateral Horn} ($R_{a,a} = 0$) retains $100.00\%$ premise support at $k_{\text{fan}} = 8$, whereas passive \textbf{TE-BPHD} exhibits self-dissipative attenuation ($52.01\%$ at $k_{\text{fan}}=1$ to $15.93\%$ at $k_{\text{fan}}=8$). On full grounded natural language theories containing reciprocal rule pairs and soft port overlap, however, this diagnostic trade-off reverses sharply (Figure~\ref{fig:passivity_wavefront}b): \textbf{Gain-Matched Unilateral Horn} crosses the two-cycle implicit resolvent pole predicted by Proposition~\ref{prop:sylvester}, driving maximum state norm to $1.66 \times 10^8$ at $K=12$ and floating-point overflow ($\infty$) at $K=24$, and collapsing locked test \texttt{depth-5} counterfactual consistency by $-50.00$ percentage points relative to \textbf{TE-BPHD} ($29.17 \pm 30.03\%$ vs.\ $79.17 \pm 5.48\%$, $95\%$ CI $[+44.67, +55.17]$, $p_{\text{Holm}} = 0.0050$). Under \textbf{TE-BPHD}, $K=24$ Euclidean norm remains tightly bounded at $4.49 \pm 0.50$ while converse \texttt{Unknown} recall reaches $98.96 \pm 0.90\%$, and causal fixed-weight knockouts ($K=0$ readout or $J = R_{\text{rule}} = 0$ null-rule flow) collapse development counterfactual consistency from $90.67 \pm 3.82\%$ to $0.17 \pm 0.29\%$, verifying that multi-hop trajectory evolution carries the deductive signal (Figure~\ref{fig:passivity_wavefront}c).

\begin{figure*}[t]
\centering
\includegraphics[width=\textwidth]{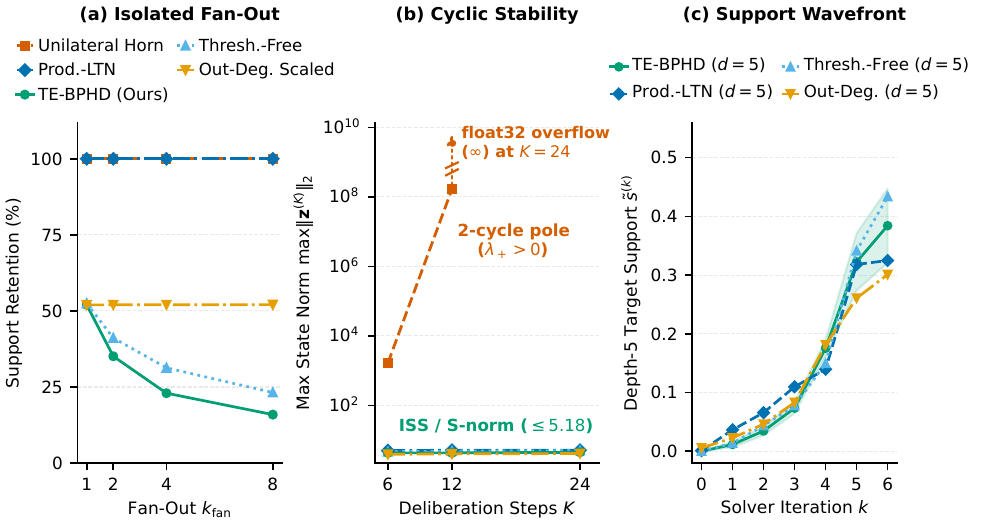}
\caption{\textbf{Sylvester Passivity, Two-Cycle Resolvent Poles, and Multi-Hop Wavefront Dynamics.} \textbf{(a)} Controlled isolated fan-out retention across $k_{\text{fan}} \in \{1, 2, 4, 8\}$ shared-premise rules. \textbf{(b)} Deliberation horizon extrapolation ($K \in \{6, 12, 24\}$, log scale), contrasting two-cycle resolvent blow-up in non-passive \textbf{Gain-Matched Unilateral Horn} ($\lambda_+ = +w_r/2 > 0$, terminating at $1.66 \times 10^8$ at $K=12$ prior to \texttt{float32} overflow $\infty$ at $K=24$) against bounded trajectories under \textbf{TE-BPHD} ($4.49$) and \textbf{Grounded Product-T-Norm LTN} ($5.18$). \textbf{(c)} Iteration-by-iteration ($k = 0 \dots 6$) target support wavefront $\tilde{s}^{(k)}$ on \texttt{depth-5} provable queries (shaded band: $\pm 1\,\text{SD}$ across seeds).}
\label{fig:passivity_wavefront}
\end{figure*}

Cross-plugging \textbf{Grounded Product-T-Norm LTN} and \textbf{Threshold-Free TE-BPHD} bounds the scope of Port-Hamiltonian mechanisms without overclaiming. First, removing the sigmoid firing threshold (\textbf{Threshold-Free TE-BPHD}) yields $77.33 \pm 4.91\%$ locked test \texttt{depth-5} counterfactual consistency, with no statistically significant difference detected relative to thresholded \textbf{TE-BPHD} ($\Delta = +1.83$ percentage points, $p_{\text{Holm}} = 1.0000$), while thresholded gating contributes a modest $+2.23$ percentage point gain on six-domain OOD ($94.95\%$ vs.\ $92.72\%$, $p_{\text{Holm}} = 0.0050$). Second, restricting states to $[0, 1]^{2N}$ via Probabilistic-Sum S-norm saturation (\textbf{Grounded Product-T-Norm LTN}) prevents cyclic divergence without Woodbury matrix solves ($K=24$ norm $5.18 \pm 0.08$, $8.15\,\text{ms/query}$ batch-$16$ throughput on \texttt{NVIDIA L4}), achieving $94.50 \pm 0.87\%$ query-weighted OOD accuracy ($86.08 \pm 1.22\%$ six-domain macro-mean; no detectable query-weighted difference vs.\ \textbf{TE-BPHD} at $p_{\text{Holm}} = 1.0000$) and $73.67 \pm 8.81\%$ \texttt{depth-5} counterfactual consistency ($\Delta = +5.50$ percentage points, raw $p = 0.0169$, not familywise-significant at $p_{\text{Holm}} = 0.2197$), while \textbf{TE-BPHD} retains a $+2.93$ percentage point advantage on standard \texttt{depth-5} test accuracy ($86.96\%$ vs.\ $84.03\%$, $p_{\text{Holm75}} = 0.0108$) and tighter cross-seed variance ($5.48\%$ vs.\ $8.81\%$ counterfactual SD). Thus, symmetric single-rule passivity eliminates positive two-cycle modes in the tested additive linear implicit solver, whereas bounded S-norm saturation provides a viable inversion-free alternative over the same grounded bilateral port architecture. Upstream in Table~\ref{tab:dev_factorial}, the token specificity gap (\textbf{Spec.\ Gap}, measuring mean specificity on open-class content words minus held-out unseen connectives) registers $-9.67 \pm 1.94$ percentage points under rolling corpus-IDF gating versus $+66.66$ percentage points under the stateless specificity MLP, a vocabulary-novelty diagnostic examined further in Section~\ref{sec:discussion}.

\subsection{Linguistic Regime Trade-Off: Clause-Factored Grounding vs.\ Multi-Fact Prose}

\begin{table*}[t]
\centering
\footnotesize
\setlength{\tabcolsep}{2.8pt}
\caption{\textbf{Regime II Joint Multi-Style Supervision on the Locked Test Suite ($3$ Matched Seeds).}}
\label{tab:regime2_joint}
\begin{tabular}{@{}lcccccccc@{}}
\toprule
\textbf{Jointly Supervised Model} & \textbf{Dev PR CF $\uparrow$} & \textbf{D5 Std $\uparrow$} & \textbf{D5 CF $\uparrow$} & \textbf{PR Std $\uparrow$} & \textbf{PR CF $\uparrow$} & \textbf{OOD Q-Wt $\uparrow$} & \textbf{$p_{\text{Holm}}$} & \textbf{$p_{\text{Holm}}$} \\
 & \textbf{($948$)} & \textbf{($1{,}958$)} & \textbf{($600$)} & \textbf{($8{,}008$)} & \textbf{($1{,}200$)} & \textbf{($5{,}270$)} & \textbf{(PR CF)} & \textbf{(OOD)} \\
\midrule
\textbf{TE-BPHD Joint (1-Slot)} & $9.56 \pm 3.17$ & $70.99 \pm 16.36$ & $41.00 \pm 36.10$ & $51.12 \pm 2.61$ & $12.67 \pm 4.70$ & $\mathbf{86.42 \pm 2.90}$ & --- & --- \\
Multi-Slot OT Joint (5-Slot) & $28.69 \pm 4.86$ & $51.67 \pm 9.22$ & $23.33 \pm 17.96$ & $57.76 \pm 2.55$ & $34.11 \pm 5.01$ & $72.76 \pm 17.29$ & $0.0050$ & $0.0050$ \\
Seq2Seq T5-small Joint (16 ep) & $76.09 \pm 0.87$ & $78.24 \pm 0.72$ & $54.33 \pm 1.89$ & $89.62 \pm 0.20$ & $75.00 \pm 2.18$ & $63.65 \pm 12.43$ & $0.0050$ & $0.0050$ \\
Rec.-Depth T5 Joint ($K{=}6$, 12 ep) & $\mathbf{76.51 \pm 2.58}$ & $\mathbf{87.18 \pm 1.07}$ & $\mathbf{65.83 \pm 4.75}$ & $\mathbf{93.00 \pm 0.55}$ & $\mathbf{78.11 \pm 2.55}$ & $65.60 \pm 1.15$ & $0.0050$ & $0.0050$ \\
\bottomrule
\end{tabular}
\vspace{2pt}\par\noindent\parbox{\textwidth}{\scriptsize \textit{Notes:} Jointly trained on \texttt{depth-5} ($10{,}477$ queries) and crowdsourced \texttt{ParaRules} ($19{,}614$ queries) across seeds $\{42, 123, 910\}$ (accuracy \% $\text{mean} \pm \text{SD}$). \textbf{D5}: \texttt{depth-5}; \textbf{PR}: \texttt{ParaRules}. $p_{\text{Holm}}$ columns report $50$-test Holm--Bonferroni adjusted $p$-values relative to \textbf{TE-BPHD Joint (1-Slot)} on locked test \texttt{ParaRules} All-4 Counterfactual (\textbf{PR CF}) and six-domain \texttt{birds-electricity} query-weighted accuracy (\textbf{OOD}). Seed-level trajectories are reported in Supplementary Table~S3.}
\end{table*}

\begin{figure*}[!htbp]
\centering
\includegraphics[width=\textwidth]{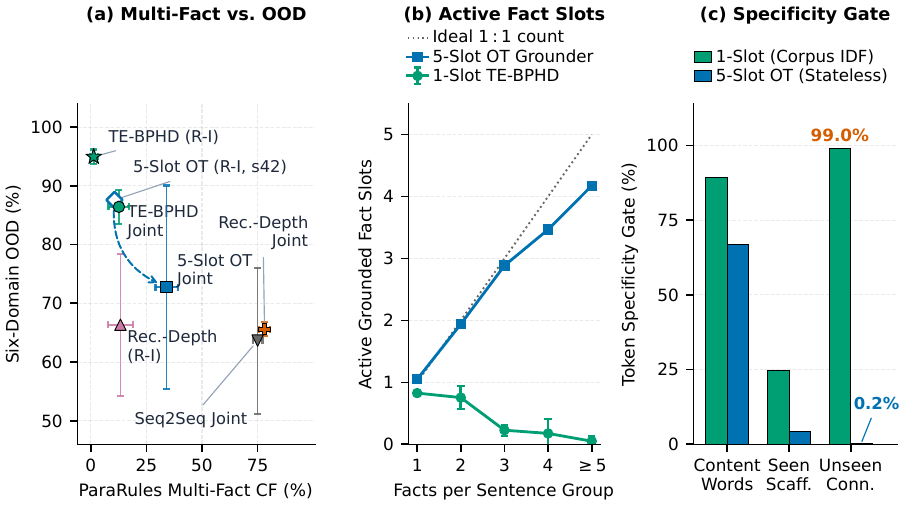}
\caption{\textbf{Multi-Fact Prose vs.\ Cross-Domain OOD Trade-Off and Grounder Slot Diagnostics.} \textbf{(a)} Locked test \texttt{ParaRules} All-4 counterfactual consistency vs.\ six-domain \texttt{birds-electricity} OOD accuracy across Regime~I (\texttt{depth-5} supervision) and Regime~II (joint multi-style supervision; horizontal/vertical error bars denote $3$-seed $\pm 1\,\text{SD}$, open diamond denotes single-seed \texttt{5-Slot OT (R-I, s42)}, and the dashed arrow links Regime~I to Regime~II training regimes). \textbf{(b)} Active grounded fact slots per sentence group across ground-truth fact multiplicity buckets ($n = 1{,}320;\; 317;\; 269;\; 169;\; 86$ sentences for $1, 2, 3, 4, \ge 5$ facts, where the $\ge 5$ bucket spans $5\text{--}6$ facts with mean $5.25$). \textbf{(c)} Token specificity gate activations across open-class content words ($5{,}572$ tokens), seen syntactic scaffolding, and held-out unseen connectives ($388$ tokens across $80$ unique paraphrase theories).}
\label{fig:pararules_slots}
\end{figure*}

Evaluating models on crowdsourced multi-fact \texttt{ParaRules} prose delineates the upstream multi-fact grounding boundary of query-only clause grounding (Table~\ref{tab:regime2_joint} and Figure~\ref{fig:pararules_slots}). When trained solely on single-clause \texttt{depth-5} sentences (Regime~I, Table~\ref{tab:locked_test_main}), single-slot \textbf{TE-BPHD} attains $1.33 \pm 1.76\%$ zero-shot All-4 counterfactual consistency on \texttt{ParaRules} because compound narrative sentences pack up to six facts into a single clause, driving the single-fact gate open fraction down to $0.043$ (Figure~\ref{fig:pararules_slots}b) whereas $0$-shot $2{,}048$-token-capped \textbf{Qwen3-1.7B} reaches $30.00\%$ on the $600$-query slice. Under Regime~I single-seed screening (seed $42$), composite \textbf{Multi-Slot OT Grounder} ($5$ Sinkhorn fact slots, stateless specificity MLP, and $3$-premise rule bodies) lifts zero-shot \texttt{ParaRules} counterfactual consistency from $0.67\%$ (seed-$42$ \textbf{TE-BPHD}; Supplementary Section~S3 and Supplementary Table~S3) to $10.67\%$ and exhibits monotonic active fact-slot scaling from $1.05$ slots on single-fact sentences to $4.17$ slots on sentences containing five or more facts (Figure~\ref{fig:pararules_slots}b), alongside a seed-$42$ rank reversal on \texttt{depth-5} counterfactuals between development ($45.00\%$ vs.\ $86.50\%$) and test ($78.00\%$ vs.\ $73.00\%$, paired sign-flip $p = 0.2824$).

Under Regime~II joint supervision on \texttt{depth-5} and \texttt{ParaRules} across all three matched seeds ($\{42, 123, 910\}$, Table~\ref{tab:regime2_joint}), composite \textbf{Multi-Slot OT Grounder Joint} ($5$ Sinkhorn slots, stateless specificity MLP, and $3$-premise rule bodies) raises locked test \texttt{ParaRules} counterfactual consistency from $12.67 \pm 4.70\%$ (\textbf{TE-BPHD Joint}) to $34.11 \pm 5.01\%$ ($\Delta = +21.44$ percentage points, $95\%$ CI $[+17.67, +25.22]$, $p_{\text{Holm}} = 0.0050$), confirming that multi-slot optimal-transport grounding partially relieves the single-fact bottleneck while trading off \texttt{depth-5} counterfactual accuracy ($23.33\%$ vs.\ $41.00\%$, where single-slot seed $910$ also exhibits initialization instability at $0.00\%$ CF; Supplementary Section~S3 and Supplementary Table~S3) and OOD transfer ($72.76\%$ vs.\ $86.42\%$). Conversely, token-level \textbf{Seq2Seq T5-small Joint} ($75.00 \pm 2.18\%$) and \textbf{Recurrent-Depth T5 Joint} ($78.11 \pm 2.55\%$) achieve strong in-distribution multi-fact prose accuracy through unconstrained cross-clause attention, yet drop by $22.78$ and $20.82$ percentage points below \textbf{TE-BPHD Joint} on zero-shot six-domain OOD transfer ($63.65 \pm 12.43\%$ and $65.60 \pm 1.15\%$ vs.\ $86.42 \pm 2.90\%$, $p_{\text{Holm}} = 0.0050$; Figure~\ref{fig:pararules_slots}a). Across all $50$ overlapping primary comparisons between the $50$-test primary and $75$-test five-suite omnibus Holm--Bonferroni families, zero $\alpha = 0.05$ significance decisions flip (Supplementary Section~S2 and Supplementary Table~S2).

\section{Discussion, Limitations, and Conclusion}
\label{sec:discussion}

\subsection{Mechanistic Synthesis Across Grounding and Deliberation}
Decomposing neuro-symbolic OWA deduction into upstream bilateral grounding ($F$) and downstream state propagation ($B$) clarifies two longstanding ambiguities in neural reasoning. First, counterfactual brittleness and OOD degradation in sequence and latent-recurrent transformers do not stem from parameter count or shallow recurrence alone: $K=6$ \textbf{Recurrent-Depth T5} and $2{,}048$-token-capped \textbf{Qwen3-1.7B} remain $28.58$ and $14.44$ percentage points behind \textbf{TE-BPHD} on zero-shot six-domain OOD transfer (full suite and paired slice) despite higher token interaction capacity. Enforcing a query-independent bilateral atom-edge bottleneck ($D=1{,}440$) blocks direct token-level shortcut paths between rule heads and queries, supporting explicit forward signed Modus Ponens composition over bilateral positive and negated literals via $\Pi_{\text{neg}}$ (matching \texttt{ProofWriter} OWA Datalog semantics, where $P \to Q$ and $\neg Q$ leave $\neg P$ \texttt{Unknown}) at sub-$10\,\text{ms}$ batch-$16$ GPU throughput. Second, downstream Port-Hamiltonian structure plays a mathematically scoped role rather than acting as a universal prerequisite: symmetric single-rule passivity ($R_r \succeq 0$, trace-minimal at $R_{a,a}=R_{c,c}=w_r/2$ on $\mathbf{a}_r \perp \mathbf{c}_r$) ensures nonsingular implicit steps for all $w_r \ge 0$, whereas omitting antecedent self-dissipation ($R_{a,a}=0$ in \textbf{Gain-Matched Unilateral Horn}) creates a positive symmetric two-cycle mode $\lambda_+ = +w_r/2 > 0$ compatible with the observed $K \in \{12, 24\}$ blow-up, while $[0,1]^{2N}$-saturating Product-T-Norm / Probabilistic-Sum aggregation achieves cycle-safe OOD transfer in $\mathcal{O}(DM)$ time by bounding support algebraically.

\subsection{Scope and Empirical Boundaries}
Four methodological and empirical scope conditions bound our study. First, under pure three-way query supervision without parse annotations, clause-factored grounding faces an observed multi-fact grounding bottleneck on crowdsourced multi-proposition prose (\texttt{ParaRules}); narrowing the gap between composite five-slot optimal-transport grounding ($34.11\%$ joint CF) and token-level recurrent attention ($78.11\%$ joint CF) while preserving OOD transfer ($86.42\%$ vs.\ $65.60\%$) and stabilizing Regime~II single-slot seed variance (seed $910$; Supplementary Table~S3) motivates hybrid discourse-segmentation pretraining. Second, while \textbf{TE-BPHD} transfers to unseen-connective paraphrases ($66.00 \pm 2.78\%$ development CF), its frozen training-corpus specificity buffer assigns $99.02\%$ specificity to unseen connectives ($-9.67 \pm 1.94$ percentage point gap), an inversion removed ($+66.66$ percentage points) by the stateless specificity head in \textbf{Multi-Slot OT Grounder} (Figure~\ref{fig:pararules_slots}c). Third, sub-clause role extraction assumes an English active-voice subject--verb--object bias, and within the $C=140$ OOD suite \texttt{BirdsVar1} and \texttt{BirdsVar2} each contain $1$ source-theory cluster ($40$ queries), motivating multi-cluster taxonomic suites and morphology-aware role pointers. Fourth, Proposition~\ref{prop:iss} establishes discrete $P$-quadratic boundedness uniformly for bounded fixed $\mathbf{u}_{\text{fact}}$ and non-negative conductances without claiming unique nonlinear fixed-point contraction across gate transitions.

\subsection{Conclusion}
We presented \textbf{TE-BPHD}, coupling clause-factored bilateral Belnap state grounding ($F$) with passivity-balanced Port-Hamiltonian deliberation ($B$) for counterfactual open-world deduction. On \texttt{ProofWriter}, explicit bilateral port grounding substantially improves counterfactual faithfulness and cross-domain OOD transfer over sequence transformers, latent recurrent-depth models, generic port-GRUs, and higher-budget ($2{,}048$-token cap) $3$-shot \textbf{Qwen3-1.7B} at far higher batch throughput, while mapping both the scoped single-rule passivity vs.\ fuzzy S-norm boundary downstream and the multi-fact prose grounding trade-off upstream.

\section*{CRediT authorship contribution statement}
\textbf{Zeyu Jia:} Conceptualization, Methodology, Software, Validation, Formal analysis, Investigation, Resources, Funding acquisition, Visualization, Project administration, Writing -- original draft, Writing -- review and editing. \textbf{Bo Yao:} Methodology, Validation, Resources, Funding acquisition, Writing -- review and editing. \textbf{Zihao Yang:} Data curation, Investigation, Writing -- review and editing.

\section*{Funding}
This work was supported by personal funds from Zeyu Jia and Bo Yao and did not receive any specific grant from funding agencies in the public, commercial, or not-for-profit sectors.

\section*{Declaration of competing interest}
The authors declare that they have no known competing financial interests or personal relationships that could have appeared to influence the work reported in this paper.

\section*{Data availability}
All experiments evaluate on the publicly available \texttt{ProofWriter V2020.12.3} dataset~\cite{tafjord2021proofwriter,proofwriter2020dataset} (\url{https://allenai.org/data/proofwriter}). Complete algebraic proofs, architectural and optimization specifications, and complete statistical verification tables are provided in Supplementary Sections~S1--S4 (Supplementary Tables~S1--S5), and complete PyTorch source code, theory-disjoint benchmark splits, and per-seed evaluation receipts are provided in the accompanying reproducibility archive (\texttt{reproducibility\_package\_20261118.zip}).

\bibliographystyle{elsarticle-num}
\bibliography{references}

@inproceedings{chen2018neural,
  title={Neural Ordinary Differential Equations},
  author={Chen, Ricky TQ and Rubanova, Yulia and Bettencourt, Jesse and Duvenaud, David K},
  booktitle={Advances in Neural Information Processing Systems (NeurIPS)},
  volume={31},
  pages={6571--6583},
  year={2018},
  doi={10.48550/arXiv.1806.07366}
}

@inproceedings{hao2024training,
  title={Training Large Language Models to Reason in a Continuous Latent Space},
  author={Hao, Shibo and Sukhbaatar, Sainbayar and Su, DiJia and Li, Xian and Hu, Zhiting and Weston, Jason and Tian, Yuandong},
  booktitle={Proceedings of the Second Conference on Language Modeling (COLM)},
  year={2025},
  doi={10.48550/arXiv.2412.06769}
}

@article{snell2024scaling,
  title={Scaling {LLM} Test-Time Compute Optimally Can Be More Effective than Scaling Model Parameters},
  author={Snell, Charlie and Lee, Jaehoon and Xu, Kelvin and Kumar, Aviral},
  journal={arXiv preprint arXiv:2408.03314},
  year={2024},
  doi={10.48550/arXiv.2408.03314}
}

@inproceedings{dehghani2018universal,
  title={Universal Transformers},
  author={Dehghani, Mostafa and Gouws, Stephan and Vinyals, Oriol and Uszkoreit, Jakob and Kaiser, {\L}ukasz},
  booktitle={International Conference on Learning Representations (ICLR)},
  year={2019},
  doi={10.48550/arXiv.1807.03819}
}

@article{vanderschaft2014port,
  title={Port-{H}amiltonian Systems Theory: An Introductory Overview},
  author={van der Schaft, Arjan and Jeltsema, Dimitri},
  journal={Foundations and Trends in Systems and Control},
  volume={1},
  number={2--3},
  pages={173--378},
  year={2014},
  doi={10.1561/2600000002}
}

@article{ortega2002interconnection,
  title={Interconnection and Damping Assignment Passivity-Based Control of Port-Controlled {H}amiltonian Systems},
  author={Ortega, Romeo and van der Schaft, Arjan and Maschke, Bernhard and Escobar, Gerardo},
  journal={Automatica},
  volume={38},
  number={4},
  pages={585--596},
  year={2002},
  doi={10.1016/S0005-1098(01)00278-3}
}

@article{desai2021port,
  title={Port-{H}amiltonian Neural Networks for Learning Explicit Time-Dependent Dynamical Systems},
  author={Desai, Shaan A and Mattheakis, Marios and Sondak, David and Protopapas, Pavlos and Roberts, Stephen J},
  journal={Physical Review E},
  volume={104},
  pages={034312},
  year={2021},
  doi={10.1103/PhysRevE.104.034312}
}

@inproceedings{rusch2022graphcon,
  title={Graph-Coupled Oscillator Networks},
  author={Rusch, T. Konstantin and Chamberlain, Benjamin P. and Rowbottom, James and Mishra, Siddhartha and Bronstein, Michael M.},
  booktitle={International Conference on Machine Learning (ICML)},
  volume={162},
  pages={18888--18909},
  year={2022},
  doi={10.48550/arXiv.2202.02296}
}

@inproceedings{greydanus2019hamiltonian,
  title={Hamiltonian Neural Networks},
  author={Greydanus, Samuel and Dzamba, Misko and Yosinski, Jason},
  booktitle={Advances in Neural Information Processing Systems (NeurIPS)},
  volume={32},
  pages={15379--15389},
  year={2019},
  doi={10.48550/arXiv.1906.01563}
}

@article{fruchart2021non,
  title={Non-Reciprocal Phase Transitions},
  author={Fruchart, Michel and Hanai, Ryo and Littlewood, Peter B. and Vitelli, Vincenzo},
  journal={Nature},
  volume={592},
  pages={363--369},
  year={2021},
  doi={10.1038/s41586-021-03375-9}
}

@inproceedings{tafjord2021proofwriter,
  title={{ProofWriter}: Generating Implications, Proofs, and Abductive Statements over Natural Language},
  author={Tafjord, Oyvind and Dalvi, Bhavana and Clark, Peter},
  booktitle={Findings of the Association for Computational Linguistics: ACL-IJCNLP 2021},
  pages={3621--3634},
  year={2021},
  doi={10.18653/v1/2021.findings-acl.317}
}

@inproceedings{clark2020transformers,
  title={Transformers as Soft Reasoners over Language},
  author={Clark, Peter and Tafjord, Oyvind and Richardson, Kyle},
  booktitle={International Joint Conference on Artificial Intelligence (IJCAI)},
  pages={3882--3890},
  year={2020},
  doi={10.24963/ijcai.2020/537}
}

@incollection{belnap1977useful,
  title={A Useful Four-Valued Logic},
  author={Belnap, Nuel D.},
  booktitle={Modern Uses of Multiple-Valued Logic},
  publisher={D. Reidel Publishing},
  pages={5--37},
  year={1977},
  doi={10.1007/978-94-010-1161-7_2}
}

@article{riegel2020logical,
  title={Logical Neural Networks},
  author={Riegel, Ryan and Gray, Alexander and Luus, Francois and Khan, Naweed and Makondo, Ndivhuwo and Akhalwaya, Ismail Yunus and Qian, Haifeng and Fagin, Ronald and Barahona, Francisco and Sharma, Udit and others},
  journal={arXiv preprint arXiv:2006.13155},
  year={2020},
  doi={10.48550/arXiv.2006.13155}
}

@article{badreddine2022logic,
  title={Logic Tensor Networks},
  author={Badreddine, Samy and d'Avila Garcez, Artur and Serafini, Luciano and Spranger, Michael},
  journal={Artificial Intelligence},
  volume={303},
  pages={103649},
  year={2022},
  doi={10.1016/j.artint.2021.103649}
}

@inproceedings{manhaeve2018deepproblog,
  title={{DeepProbLog}: Neural Probabilistic Logic Programming},
  author={Manhaeve, Robin and Dumancic, Sebastijan and Kimmig, Angelika and Demeester, Thomas and De Raedt, Luc},
  booktitle={Advances in Neural Information Processing Systems (NeurIPS)},
  volume={31},
  pages={3749--3759},
  year={2018},
  doi={10.48550/arXiv.1805.10872}
}

@inproceedings{yang2017differentiable,
  title={Differentiable Learning of Logical Rules for Knowledge Base Reasoning},
  author={Yang, Fan and Yang, Zhilin and Cohen, William W.},
  booktitle={Advances in Neural Information Processing Systems (NeurIPS)},
  volume={30},
  pages={2319--2328},
  year={2017},
  doi={10.48550/arXiv.1702.08367}
}

@article{evans2018learning,
  title={Learning Explanatory Rules from Noisy Data},
  author={Evans, Richard and Grefenstette, Edward},
  journal={Journal of Artificial Intelligence Research},
  volume={61},
  pages={1--64},
  year={2018},
  doi={10.1613/jair.5714}
}

@inproceedings{rocktaschel2017end,
  title={End-to-End Differentiable Proving},
  author={Rockt{\"a}schel, Tim and Riedel, Sebastian},
  booktitle={Advances in Neural Information Processing Systems (NeurIPS)},
  volume={30},
  pages={3788--3800},
  year={2017},
  doi={10.48550/arXiv.1705.11040}
}

@inproceedings{dong2019neural,
  title={Neural Logic Machines},
  author={Dong, Honghua and Mao, Jiayuan and Lin, Tian and Wang, Chong and Li, Lihong and Zhou, Denny},
  booktitle={International Conference on Learning Representations (ICLR)},
  year={2019},
  doi={10.48550/arXiv.1904.11694}
}

@inproceedings{locatello2020object,
  title={Object-Centric Learning with Slot Attention},
  author={Locatello, Francesco and Weissenborn, Dirk and Unterthiner, Thomas and Mahendran, Aravindh and Heigold, Georg and Uszkoreit, Jakob and Dosovitskiy, Alexey and Kipf, Thomas},
  booktitle={Advances in Neural Information Processing Systems (NeurIPS)},
  volume={33},
  pages={11525--11538},
  year={2020},
  doi={10.48550/arXiv.2006.15055}
}

@inproceedings{vinyals2015pointer,
  title={Pointer Networks},
  author={Vinyals, Oriol and Fortunato, Meire and Jaitly, Navdeep},
  booktitle={Advances in Neural Information Processing Systems (NeurIPS)},
  volume={28},
  pages={2692--2700},
  year={2015},
  doi={10.48550/arXiv.1506.03134}
}

@inproceedings{dozat2017deep,
  title={Deep Biaffine Attention for Neural Dependency Parsing},
  author={Dozat, Timothy and Manning, Christopher D.},
  booktitle={International Conference on Learning Representations (ICLR)},
  year={2017},
  doi={10.48550/arXiv.1611.01734}
}

@inproceedings{bai2019deep,
  title={Deep Equilibrium Models},
  author={Bai, Shaojie and Kolter, J. Zico and Koltun, Vladlen},
  booktitle={Advances in Neural Information Processing Systems (NeurIPS)},
  volume={32},
  pages={690--701},
  year={2019},
  doi={10.48550/arXiv.1909.01377}
}

@inproceedings{winston2020monotone,
  title={Monotone Operator Equilibrium Networks},
  author={Winston, Ezra and Kolter, J. Zico},
  booktitle={Advances in Neural Information Processing Systems (NeurIPS)},
  volume={33},
  pages={10718--10728},
  year={2020},
  doi={10.48550/arXiv.2006.08591}
}

@article{graves2016adaptive,
  title={Adaptive Computation Time for Recurrent Neural Networks},
  author={Graves, Alex},
  journal={arXiv preprint arXiv:1603.08983},
  year={2016},
  doi={10.48550/arXiv.1603.08983}
}

@inproceedings{banino2021pondernet,
  title={{PonderNet}: Learning to Ponder},
  author={Banino, Andrea and Balaguer, Jan and Blundell, Charles},
  booktitle={ICML Workshop on Automated Machine Learning},
  year={2021},
  doi={10.48550/arXiv.2107.05407}
}

@article{raffel2020exploring,
  title={Exploring the Limits of Transfer Learning with a Unified Text-to-Text Transformer},
  author={Raffel, Colin and Shazeer, Noam and Roberts, Adam and Lee, Katherine and Narang, Sharan and Matena, Michael and Zhou, Yanqi and Li, Wei and Liu, Peter J.},
  journal={Journal of Machine Learning Research},
  volume={21},
  number={140},
  pages={1--67},
  year={2020},
  doi={10.48550/arXiv.1910.10683}
}

@article{woodbury1950inverting,
  title={Inverting Modified Matrices},
  author={Woodbury, Max A.},
  journal={Memorandum Report 42, Statistical Research Group, Princeton University},
  year={1950}
}

@article{vankrieken2022analyzing,
  title={Analyzing Differentiable Fuzzy Logic Operators},
  author={van Krieken, Emile and Acar, Erman and van Harmelen, Frank},
  journal={Artificial Intelligence},
  volume={302},
  pages={103602},
  year={2022},
  doi={10.1016/j.artint.2021.103602}
}

@inproceedings{geiping2025scaling,
  title={Scaling up Test-Time Compute with Latent Reasoning: A Recurrent Depth Approach},
  author={Geiping, Jonas and McLeish, Sean and Jain, Neel and Kirchenbauer, John and Singh, Siddharth and Bartoldson, Brian R. and Kailkhura, Bhavya and Bhatele, Abhinav and Goldstein, Tom},
  booktitle={Advances in Neural Information Processing Systems (NeurIPS)},
  volume={38},
  year={2025},
  doi={10.52202/085713-1380}
}

@inproceedings{ye2025differential,
  title={Differential Transformer},
  author={Ye, Tianzhu and Dong, Li and Xia, Yuqing and Sun, Yutao and Zhu, Yi and Huang, Gao and Wei, Furu},
  booktitle={International Conference on Learning Representations (ICLR)},
  year={2025},
  doi={10.48550/arXiv.2410.05258}
}

@inproceedings{cuturi2013sinkhorn,
  title={Sinkhorn Distances: Lightspeed Computation of Optimal Transportation Distances},
  author={Cuturi, Marco},
  booktitle={Advances in Neural Information Processing Systems (NeurIPS)},
  volume={26},
  pages={2292--2300},
  year={2013}
}

@inproceedings{zhang2023unlocking,
  title={Unlocking Slot Attention by Changing Optimal Transport Costs},
  author={Zhang, Yan and Zhang, David W. and Lacoste-Julien, Simon and Burghouts, Gertjan J. and Snoek, Cees G. M.},
  booktitle={International Conference on Machine Learning (ICML)},
  volume={202},
  pages={41931--41951},
  year={2023},
  doi={10.48550/arXiv.2301.13197}
}

@article{chung2024scaling,
  title={Scaling Instruction-Finetuned Language Models},
  author={Chung, Hyung Won and Hou, Le and Longpre, Shayne and Zoph, Barret and Tay, Yi and Fedus, William and Li, Yunxuan and Wang, Xuezhi and Dehghani, Mostafa and Brahma, Siddhartha and others},
  journal={Journal of Machine Learning Research},
  volume={25},
  number={70},
  pages={1--53},
  year={2024},
  doi={10.48550/arXiv.2210.11416}
}

@inproceedings{galkin2024towards,
  title={Towards Foundation Models for Knowledge Graph Reasoning},
  author={Galkin, Mikhail and Yuan, Xinyu and Mostafa, Hesham and Tang, Jian and Zhu, Zhaocheng},
  booktitle={International Conference on Learning Representations (ICLR)},
  year={2024},
  doi={10.48550/arXiv.2310.04562}
}

@article{davydov2024non,
  title={Non-{E}uclidean Monotone Operator Theory and Applications},
  author={Davydov, Alexander and Jafarpour, Saber and Proskurnikov, Anton V. and Bullo, Francesco},
  journal={Journal of Machine Learning Research},
  volume={25},
  number={307},
  pages={1--33},
  year={2024},
  doi={10.48550/arXiv.2303.11273}
}

@inproceedings{gu2020implicit,
  title={Implicit Graph Neural Networks},
  author={Gu, Fangda and Chang, Heng and Zhu, Wenwu and Sojoudi, Somayeh and El Ghaoui, Laurent},
  booktitle={Advances in Neural Information Processing Systems (NeurIPS)},
  volume={33},
  pages={11984--11995},
  year={2020},
  doi={10.48550/arXiv.2009.06211}
}

@article{cohen2020tensorlog,
  title={{TensorLog}: A Differentiable Deductive Database},
  author={Cohen, William W. and Yang, Fan and Mazaitis, Kathryn Rivard},
  journal={Journal of Artificial Intelligence Research},
  volume={67},
  pages={285--325},
  year={2020},
  doi={10.1613/jair.1.11475}
}

@inproceedings{bergen2021systematic,
  title={Systematic Generalization with Edge Transformers},
  author={Bergen, Leon and O'Donnell, Timothy J. and Bahdanau, Dzmitry},
  booktitle={Advances in Neural Information Processing Systems (NeurIPS)},
  volume={34},
  pages={1390--1402},
  year={2021},
  doi={10.48550/arXiv.2112.00578}
}

@inproceedings{yang2025socrates,
  title={Do Large Language Models Perform Latent Multi-Hop Reasoning without Exploiting Shortcuts?},
  author={Yang, Sohee and Kassner, Nora and Gribovskaya, Elena and Riedel, Sebastian and Geva, Mor},
  booktitle={Findings of the Association for Computational Linguistics: ACL 2025},
  year={2025},
  doi={10.18653/v1/2025.findings-acl.205}
}

@inproceedings{saparov2023language,
  title={Language Models Are Greedy Reasoners: A Systematic Formal Analysis of Chain-of-Thought},
  author={Saparov, Abulhair and He, He},
  booktitle={International Conference on Learning Representations (ICLR)},
  year={2023},
  doi={10.48550/arXiv.2210.01240}
}

@inproceedings{han2024folio,
  title={{FOLIO}: Natural Language Reasoning with First-Order Logic},
  author={Han, Simeng and Schoelkopf, Hailey and Zhao, Yilun and Qi, Zhenting and Riddell, Martin and Zhou, Wenfei and Coady, James and Peng, David and Qiao, Yujie and Benson, Luke and others},
  booktitle={Proceedings of the 2024 Conference on Empirical Methods in Natural Language Processing (EMNLP)},
  pages={22017--22031},
  year={2024},
  doi={10.18653/v1/2024.emnlp-main.1229}
}

@article{wu2024neurosymbolic,
  title={Neuro-symbolic recommendation model based on logic query},
  author={Wu, Maonian and Chen, Bang and Zhu, Shaojun and Zheng, Bo and Peng, Wei and Zhang, Mingyi},
  journal={Knowledge-Based Systems},
  volume={284},
  pages={111311},
  year={2024},
  doi={10.1016/j.knosys.2023.111311}
}

@article{chen2022rulemining,
  title={Rule mining over knowledge graphs via reinforcement learning},
  author={Chen, Lihan and Jiang, Sihang and Liu, Jingping and Wang, Chao and Zhang, Sheng and Xie, Chenhao and Liang, Jiaqing and Xiao, Yanghua and Song, Rui},
  journal={Knowledge-Based Systems},
  volume={242},
  pages={108371},
  year={2022},
  doi={10.1016/j.knosys.2022.108371}
}

@article{dai2025question,
  title={Question answering over spatio-temporal knowledge graph},
  author={Dai, Xinbang and Li, Huiying and Hu, Nan and Chen, Yongrui and Jin, Rihui and Hu, Huikang and Qi, Guilin},
  journal={Knowledge-Based Systems},
  volume={329},
  pages={114314},
  year={2025},
  doi={10.1016/j.knosys.2025.114314}
}

@article{yang2025qwen3,
  title={{Qwen3} Technical Report},
  author={Yang, An and Li, Anfeng and Yang, Baosong and Zhang, Beichen and Hui, Binyuan and Zheng, Bo and Yu, Bowen and Gao, Chang and Huang, Chengen and Lv, Chenxu and others},
  journal={arXiv preprint arXiv:2505.09388},
  year={2025},
  doi={10.48550/arXiv.2505.09388}
}

@inproceedings{pan2023logiclm,
  title={{Logic-LM}: Empowering Large Language Models with Symbolic Solvers for Faithful Logical Reasoning},
  author={Pan, Liangming and Albalak, Alon and Wang, Xinyi and Wang, William Yang},
  booktitle={Findings of the Association for Computational Linguistics: EMNLP 2023},
  pages={3806--3824},
  year={2023},
  doi={10.18653/v1/2023.findings-emnlp.248}
}

@inproceedings{olausson2023linc,
  title={{LINC}: A Neurosymbolic Approach for Logical Reasoning by Combining Language Models with First-Order Logic Provers},
  author={Olausson, Theo X. and Gu, Alex and Lipkin, Benjamin and Zhang, Cedegao E. and Solar-Lezama, Armando and Tenenbaum, Joshua B. and Levy, Roger},
  booktitle={Proceedings of the 2023 Conference on Empirical Methods in Natural Language Processing (EMNLP)},
  pages={5153--5176},
  year={2023},
  doi={10.18653/v1/2023.emnlp-main.313}
}

@article{phipson2010permutation,
  title={Permutation P-values Should Never Be Zero: Calculating Exact P-values When Permutations Are Randomly Drawn},
  author={Phipson, Belinda and Smyth, Gordon K.},
  journal={Statistical Applications in Genetics and Molecular Biology},
  volume={9},
  number={1},
  pages={Article 39},
  year={2010},
  doi={10.2202/1544-6115.1585}
}

@misc{proofwriter2020dataset,
  title={{ProofWriter} Dataset ({V2020.12.3})},
  author={Tafjord, Oyvind and Dalvi, Bhavana and Clark, Peter},
  howpublished={Allen Institute for AI, \url{https://allenai.org/data/proofwriter}},
  year={2020},
  doi={10.18653/v1/2021.findings-acl.317},
  note={[dataset]}
}

\end{document}